\documentclass[10pt,twocolumn,letterpaper]{article}

\usepackage{amsthm} 
\newtheorem{theorem}{Theorem}
\newtheorem{lemma}{Lemma}
\usepackage[pagenumbers]{cvpr} 

\usepackage{multirow} 
\definecolor{cvprblue}{rgb}{0.21,0.49,0.74}
\usepackage[pagebackref,breaklinks,colorlinks,allcolors=cvprblue]{hyperref}

\def\paperID{}
\def\confName{CVPR}
\def\confYear{2026}

\title{PFGNet: A Fully Convolutional Frequency-Guided Peripheral Gating Network for Efficient Spatiotemporal Predictive Learning}

\author{
    Xinyong Cai$^{1}$ \quad Changbin Sun$^{1}$ \quad Yong Wang$^{2}$ \quad Hongyu Yang$^{1}$ \quad Yuankai Wu$^{1}$ \thanks{Corresponding author} \\
    $^{1}$College of Computer Science, Sichuan University \\
    $^{2}$Department of Data and Systems Engineering, The University of Hong Kong \\
    {\tt\small fenghuajuedai09@gmail.com, wuyk0@scu.edu.cn}
}

\begin{document}
\maketitle
\begin{abstract}
Spatiotemporal predictive learning (STPL) aims to forecast future frames from past observations and is essential across a wide range of applications. Compared with recurrent or hybrid architectures, pure convolutional models offer superior efficiency and full parallelism, yet their fixed receptive fields limit their ability to adaptively capture spatially varying motion patterns. Inspired by biological center--surround organization and frequency-selective signal processing, we propose \textbf{PFGNet}, a fully convolutional framework that dynamically modulates receptive fields through pixel-wise frequency-guided gating. The core Peripheral Frequency Gating (PFG) block extracts localized spectral cues and adaptively fuses multi-scale large-kernel peripheral responses with learnable center suppression, effectively forming spatially adaptive band-pass filters. To maintain efficiency, all large kernels are decomposed into separable 1D convolutions ($1\times k$ followed by $k\times1$), reducing per-channel computational cost from $\mathcal{O}(k^2)$ to $\mathcal{O}(2k)$. PFGNet enables structure-aware spatiotemporal modeling without recurrence or attention. Experiments on Moving MNIST, TaxiBJ, Human3.6M, and KTH show that PFGNet delivers SOTA or near-SOTA forecasting performance with substantially fewer parameters and FLOPs. Our code is available at \href{https://github.com/fhjdqaq/PFGNet}{https://github.com/fhjdqaq/PFGNet}.
\end{abstract}    
\section{Introduction}
\label{sec:intro}

\begin{quote}
\textit{``We don’t just passively perceive the world; we actively generate it.''} --- Dr.~Anil~Seth
\end{quote}

This profound insight from cognitive science reveals that {perception is inherently predictive}. The human visual system does not passively register light; it \textit{anticipates} spatial structure to construct coherent representations~\cite{clark2015surfing}. This predictive process relies on a hierarchy of mechanisms that selectively amplify informative signals while suppressing redundancy, as exemplified by {center--surround suppression} in the retina and primary visual cortex~\cite{hubel1962receptive,kuffler1953discharge}. Mathematically, this antagonistic receptive field acts as a {spatial band-pass filter}, selectively responding to mid-frequency changes (edges, textures) while suppressing low-frequency (uniform) and high-frequency (noise) components~\cite{marr1980theory}. Such biological efficiency inspires a critical question: \textit{can we design computational operators that emulate this predictive, frequency-selective mechanism?}

\begin{figure}[t]
\centering
\includegraphics[width=0.95\linewidth]{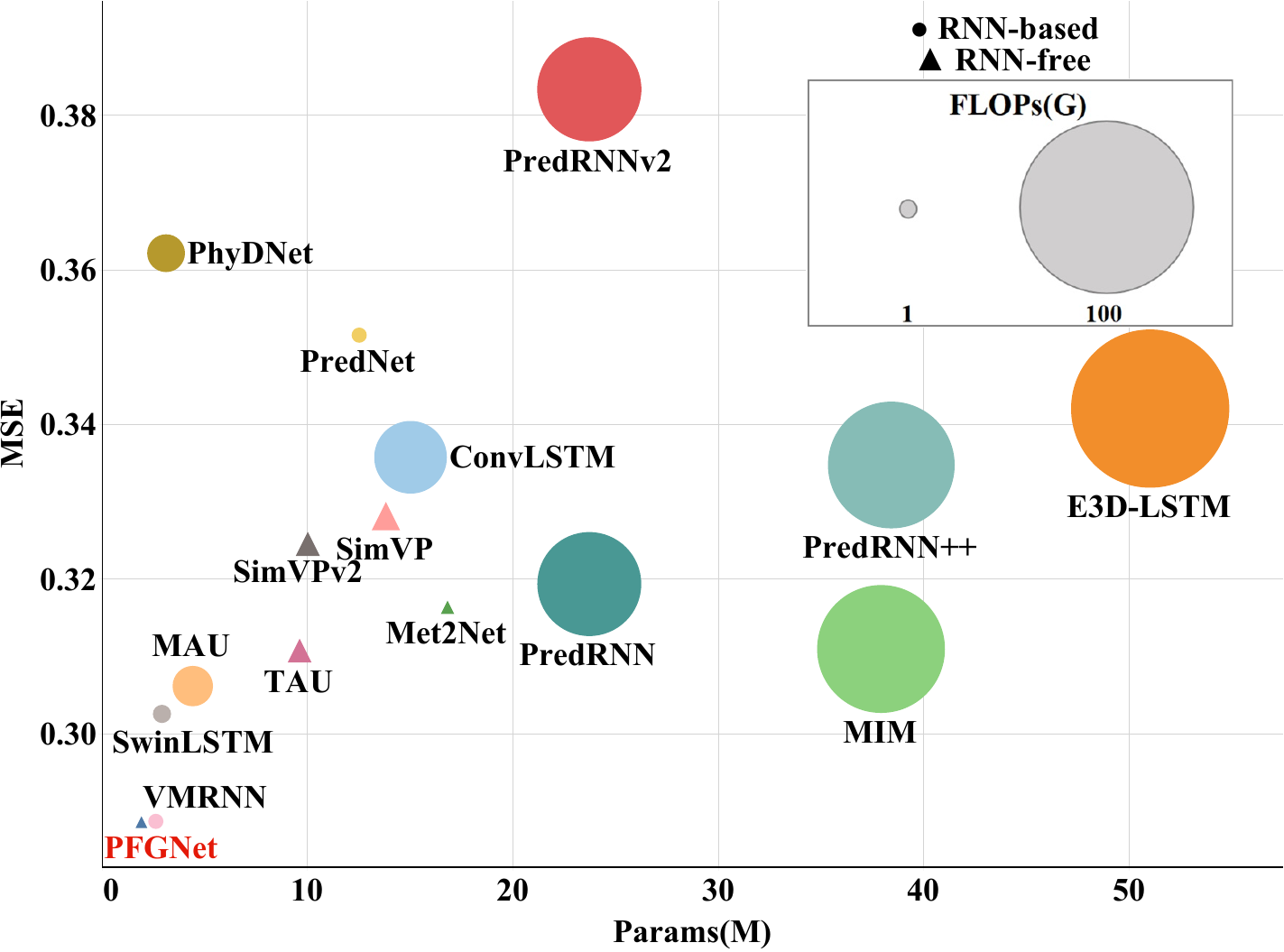}
\caption{Performance-efficiency trade-off on TaxiBJ. Bubble size denotes FLOPs. PFGNet achieves SOTA MSE with minimal cost.}
\label{fig:taxibj_bubble}
\vspace{-1.6em}
\end{figure}

Spatiotemporal predictive learning (STPL) mirrors this biological imperative: models must anticipate future frames from past observations in an unsupervised manner. STPL powers diverse applications---including weather nowcasting~\cite{shi2015convolutional,wang2017predrnn,li2025met2net}, autonomous driving~\cite{bhattacharyya2018long,lotter2017deep,Kwon_2019_CVPR}, traffic flow prediction~\cite{zhang2017deep,fang2019gstnet, cheng2024rethinking}, and human motion anticipation~\cite{wang2018rgb,babaeizadeh2017stochastic}---by learning latent spatial correlations and temporal dynamics from raw sequences. Among existing paradigms, {pure convolutional models} (e.g., SimVP~\cite{gao2022simvp}, TAU~\cite{tan2023temporal}, STLight~\cite{alfarano2025stlight}) offer superior parallelism and scalability over recurrent counterparts (e.g., ConvLSTM~\cite{shi2015convolutional}, PredRNN~\cite{wang2017predrnn}, SwinLSTM~\cite{tang2023swinlstm}, VMRNN~\cite{tang2024vmrnn}), yet remain constrained by {fixed, uniform receptive fields} that fail to adapt to spatially varying motion patterns.

Recent advances in {large-kernel ConvNets}~\cite{ding2022scaling,ding2024unireplknet,liu2022more} demonstrate that sufficiently wide receptive fields enable CNNs to approximate global context without depth, bridging local inductive bias with long-range dependency modeling. Recent studies~\cite{ding2024unireplknet, li2025shiftwiseconv, zhang2025scaling} further show that large convolutions can be reliably decomposed into cascaded small kernels without notable degradation in representational power, indicating that the benefits of large receptive fields need not rely on monolithic kernels. 

However, {simply enlarging kernels is insufficient}. Optimal receptive field size varies \textit{pixel-wise}: textured regions demand broad integration, while homogeneous areas require suppression of redundant low-frequency responses. This spatial heterogeneity echoes {center--surround organization}, suggesting that large kernels should be \textit{dynamically modulated} to mimic biological opponency. From this perspective, we make a key observation: {a large-kernel convolution with learnable center suppression mathematically recovers a ring-shaped (annular) band-pass filter~\cite{gonzalez2009digital}}. Consider a large kernel $H_L$ and a small central kernel $H_S$. Their difference $H_L - \beta H_S$ yields:
\begin{equation}
(H_L - \beta H_S) \ast x = \text{peripheral response} - \beta \cdot \text{central response}.
\end{equation}
This is similar to the {Difference of Gaussians (DoG)}---a classic model of center--surround receptive fields~\cite{marr1980theory}---and, in the frequency domain, implements a {ring-shaped band-pass filter} that amplifies mid-frequencies while attenuating DC (low) and high-frequency noise. Prior works~\cite{turner2018receptive,pan2023generalizing} show such filters dominate early visual processing; we now show they can be learned and adaptively gated in CNNs.

Building on this insight, we introduce {PFGNet}, a fully convolutional framework with a novel {Peripheral Frequency Gating (PFG)} block. PFG dynamically modulates multi-scale large-kernel responses using {pixel-wise frequency cues}---gradient magnitude, Laplacian, and local variance---extracted via lightweight depthwise filters. These cues drive a gating network that:
\begin{enumerate}
    \item {Suppresses central redundancy} via learnable $\beta$-scaled inhibition,
    \item {Fuses peripheral responses} across scales (e.g., $\mathcal{K}=\{9,15,31\}$),
    \item {Forms adaptive spatial band-pass filters} tailored to local texture.
\end{enumerate}
Combined with GLU-style~\cite{dauphin2017language} channel mixing and GRN normalization~\cite{woo2023convnext}, PFG enables {structure-aware, frequency-selective spatiotemporal modeling} without recurrence or attention. As shown in Figure~\ref{fig:taxibj_bubble}, PFGNet achieves {state-of-the-art forecasting accuracy} on TaxiBJ with {significantly fewer parameters and FLOPs} than recurrent, hybrid, or small-kernel baselines.

In summary, our contributions are threefold:  
\begin{enumerate}
    \item \textbf{Biological grounding}: We formalize large-kernel center suppression as a learnable ring-shaped band-pass filter, linking CNNs to center--surround receptive fields.
    \item \textbf{Frequency-guided adaptivity}: PFG uses local spectral cues to achieve pixel-wise receptive field modulation.
    \item \textbf{Scalable efficiency}: PFGNet delivers efficient, high-accuracy STPL via pure convolution and {fully separable large kernels}: \emph{every} $k \times k$ convolution is decomposed into a $1 \times k$ horizontal followed by a $k \times 1$ vertical convolution, reducing per-channel complexity from $\mathcal{O}(k^2)$ to $\mathcal{O}(2k)$.
\end{enumerate}

\section{Related Work}
\label{sec:rw}

\subsection{Spatiotemporal Predictive Learning}

\textbf{Recurrent-based models} leverage memory units to capture temporal dynamics sequentially. Starting from ConvLSTM~\cite{shi2015convolutional}, the PredRNN family~\cite{wang2017predrnn,wang2018predrnn++} introduces spatiotemporal memory flow and causal LSTMs to enhance long-term dependency modeling. Subsequent works further refine this paradigm: MIM~\cite{wang2019memory} mitigates non-stationarity via memory differencing; E3D-LSTM~\cite{wang2018eidetic} integrates 3D convolutions; PhyDNet~\cite{guen2020disentangling} disentangles physical constraints; MAU~\cite{chang2021mau} and CrevNet~\cite{yu2020efficient} improve motion encoding and invertibility. Recent hybrid designs inject global context into recurrent cells: SwinLSTM~\cite{tang2023swinlstm} replaces convolution with shifted-window self-attention; VMRNN~\cite{tang2024vmrnn} integrates Vision Mamba’s selective SSMs for efficient long-sequence modeling. Despite strong temporal expressiveness, these models suffer from {autoregressive inference}, {limited parallelism}, and {high latency}~\cite{tan2023openstl}.

\noindent\textbf{Recurrent-free convolutional models} eliminate recurrence for full parallelization and scalability. SimVP~\cite{gao2022simvp} proposes an encoder–translator–decoder pipeline, showing that a strong spatial backbone can implicitly model temporal evolution. Its variants replace the translator with MetaFormer architectures~\cite{yu2022metaformer,tan2023openstl}, suggesting that {backbone design dominates performance}. Recent works explore structural priors: MMVP~\cite{zhong2023mmvp} decouples motion and appearance; PastNet~\cite{wu2024pastnet} incorporates physical knowledge; TAT~\cite{nie2024triplet} and WaST~\cite{nie2024wavelet} use multi-dimensional attention and wavelet decomposition. Met2Net~\cite{li2025met2net} targets multivariate forecasting with latent prediction constraints. While efficient, these models rely on {fixed receptive fields}, failing to adapt to spatially varying motion patterns.

\subsection{Large-Kernel Convolutional Networks}

The resurgence of large-kernel ConvNets addresses the receptive field limitation of standard CNNs. ViT~\cite{dosovitskiy2021an} and Swin Transformer~\cite{liu2021swin} demonstrate that global context is key to performance, which small kernels struggle to capture efficiently. RepLKNet~\cite{ding2022scaling} scales kernels to $31\!\times\!31$ via re-parameterization; SLaK~\cite{liu2022more} pushes to $51\!\times\!51$ using sparse factorization; UniRepLKNet~\cite{ding2024unireplknet} unifies design principles: \textit{``see wide without going deep''}. These works prove that large kernels approximate global attention with {linear complexity} and {inductive bias}. However, they apply {uniform kernels}, ignoring pixel-wise variation in optimal receptive field size.

\subsection{Frequency-Domain and Peripheral Modeling}

Frequency-domain methods enable efficient global mixing by operating in transformed spaces. AFNO~\cite{guibas2021efficient} performs token mixing in Fourier space with learnable filtering; DCFormer~\cite{li2023discrete} models directly in the DCT domain with frequency-component selection; Octave Convolution~\cite{chen2019drop} partitions features into high- and low-frequency groups and computes at lower resolution for the low-frequency branch, enlarging the effective receptive field while reducing cost. {However, these operators require explicit transformations into the frequency domain, which introduce additional computational overhead and memory access costs.} In contrast, spatial-domain approaches avoid such transformations: PeLK~\cite{chen2024pelk} proposes \textit{peripheral convolution} with parameter sharing to scale kernels beyond $100\!\times\!100$; PerViT~\cite{min2022peripheral} introduces eccentricity-aware positional encoding to mimic center--surround attention bias.

Despite progress, existing works lack {pixel-wise frequency guidance} and {explicit center suppression}. Channel-level or band-level gating cannot adapt to local texture; uniform large kernels waste computation on homogeneous regions. No prior work unifies {biological center--surround}, {frequency-domain filtering}, and {adaptive large-kernel fusion} in a \textit{pure convolutional} STPL framework.
\section{Methodology}
\label{sec:method}

\begin{figure*}[t]
    \centering
    \includegraphics[scale=0.49]{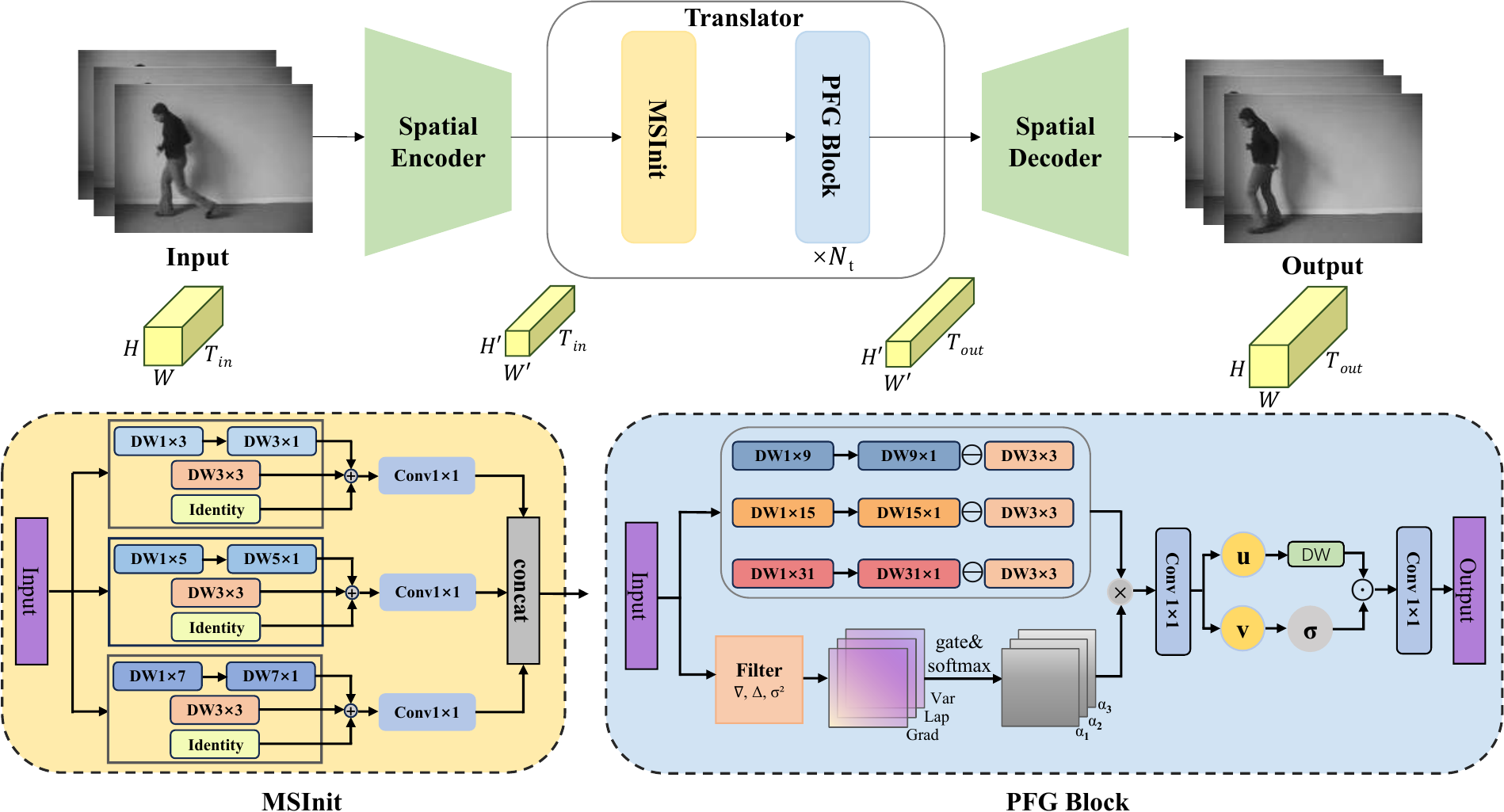}
    \caption{\textbf{Overall architecture and core modules of PFGNet.} The model follows a SimVP-style encoder--translator--decoder pipeline. The input sequence \(\{\mathbf{I}_t\}_{t=1}^{T_{\mathrm{in}}}\) is encoded into latent features, temporally packed, and processed by a MSInit followed by \(N_t\) PFG blocks.}
    \label{fig:architecture1}
\end{figure*}

\subsection{Overall Architecture}

PFGNet follows a SimVP-style encoder--translator--decoder~\cite{gao2022simvp} pipeline. We use \(N_s\) to denote the spatial encoder/decoder depth measured by the number of down- and up-sampling blocks, and \(N_t\) to denote the number of PFG blocks in the middle module.
Given input frames \(\{\mathbf{I}_t \in \mathbb{R}^{C_{\mathrm{in}}\times H \times W}\}_{t=1}^{T_{\mathrm{in}}}\), a shared spatial encoder extracts feature maps \(\mathbf{F}_t = \mathrm{Enc}(\mathbf{I}_t) \in \mathbb{R}^{C \times H' \times W'}\), which are then packed along the temporal dimension into \(\mathbf{Z} \in \mathbb{R}^{C' \times H' \times W'}\) with \(C'=T_{\mathrm{in}}\!\cdot\!C\). The middle stage starts with an MSInit (multi-scale initialization) module that generates coarse low-, mid-, and high-frequency components via identity-preserving mid- and small-kernel paths.
This is followed by \(N_t\) stacked PFG blocks, each dynamically selecting among multi-scale large-kernel peripheral responses using pixel-wise frequency-guided gating (Sobel gradient, Laplacian, local variance) with center suppression. The output \(\mathbf{Z}'\) is unpacked temporally to recover
\(\{\mathbf{F}'_t\}_{t=1}^{T_{\mathrm{out}}}\), which are then fed into a symmetric decoder to reconstruct the output frames \(\{\mathbf{O}_t\}_{t=1}^{T_{\mathrm{out}}}\). See Figure~\ref{fig:architecture1} for details.

\subsection{Lightweight Multi-Scale Initialization}
\label{sec:msinit}
The PFG block selects at each pixel the most relevant response among several scales through frequency-guided gating. For this content-adaptive fusion to work well, the model must first access diverse multi-scale responses covering low, mid, and high frequencies. Without such structured initialization, PFG would either work on uniform small-kernel features with limited long-range context or trigger costly full-scale processing, both of which are suboptimal.

To build this foundation efficiently, we introduce \textbf{MSInit}. It performs an initial condensation of features and provides stable multi-scale responses using separable one-dimensional kernels that approximate a \(k_m \times k_m\) response. For each scale \(m \in \{1,\dots,M\}\):
\begin{equation}
T_m(\mathbf{Z}) = \mathbf{v}_m \ast (\mathbf{h}_m \ast \mathbf{Z}) + d_m(\mathbf{Z}) + \mathbf{Z},
\label{eq:msinit_branch}
\end{equation}
where \(\mathbf{h}_m \in \mathbb{R}^{1 \times k_m}\) and \(\mathbf{v}_m \in \mathbb{R}^{k_m \times 1}\) are one-dimensional kernels, \(d_m(\cdot)\) denotes a \(3\times3\) depthwise convolution branch enhancing mid-frequency sensitivity, and the identity term \(\mathbf{Z}\) maintains gradient flow. Outputs are projected by \(1\times1\) convolutions with channels evenly divided across branches, and then concatenated to form the multi-scale tensor:
\begin{equation}
\mathbf{X} = \mathrm{Concat}\big(\{\mathrm{Conv}_{1\times1}(T_m(\mathbf{Z}))\}_{m=1}^{M}\big) \in \mathbb{R}^{C' \times H' \times W'}.
\label{eq:msinit_out}
\end{equation}
In practice, we use \(k_m \in \{3,5,7\}\), yielding strong performance and reliable multi-scale context for subsequent frequency-guided gating.

\subsection{Pixel-wise Frequency-Guided Peripheral Gating}
\label{sec:pfg}

The core of PFGNet is the \textbf{PFG block} (Figure~\ref{fig:pfg_block}), which dynamically fuses multi-scale large-kernel peripheral responses using pixel-wise frequency cues.

\begin{figure}[t]
\centering
\includegraphics[width=1.03\linewidth]{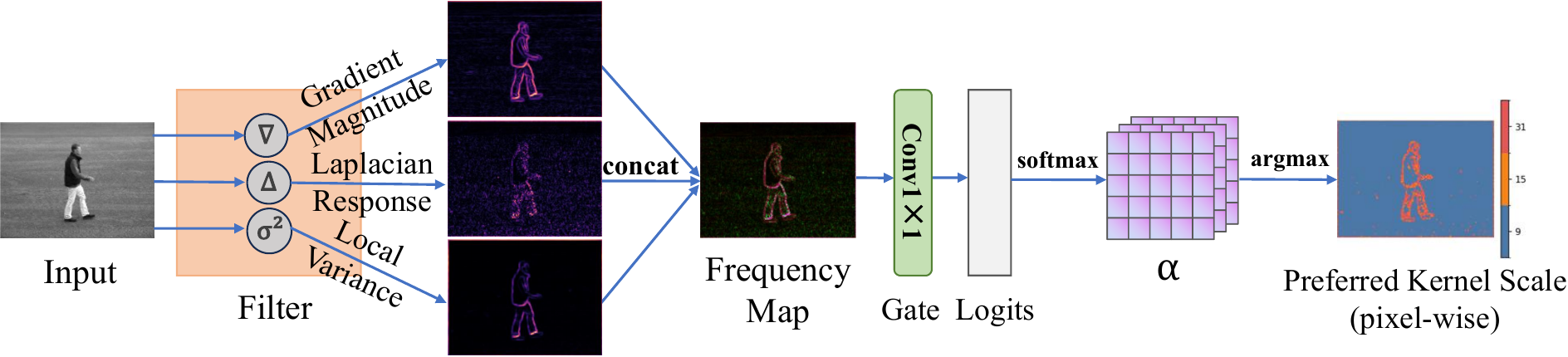}
\caption{
\textbf{Frequency-Guided Peripheral Gating in the PFG block.} Three local spectral cues (gradient magnitude, Laplacian, local variance) are extracted via fixed depthwise filters, channel-averaged, concatenated into a 3-channel frequency map, and passed through a \(1\!\times\!1\) conv to produce per-pixel gate logits. Softmax over the \(K\) scales yields selection weights \(\boldsymbol{\alpha}_k\) (the argmax visualization is only used to illustrate the preferred scale).}
\label{fig:pfg_block}
\end{figure}

From input \(\mathbf{X} \in \mathbb{R}^{C' \times H' \times W'}\), we compute three spectral cues (frequency descriptor) via fixed depthwise convolution followed by channel averaging:
\begin{equation}
\begin{aligned}
\mathbf{f}_1 &= \sqrt{(\mathbf{G}_x \ast \mathbf{X})^2 + (\mathbf{G}_y \ast \mathbf{X})^2}, \\
\mathbf{f}_2 &= |\mathbf{L} \ast \mathbf{X}|, \\
\mathbf{f}_3 &= \mathbb{E}_{3\!\times\!3}[\mathbf{X}^2] - \mathbb{E}_{3\!\times\!3}[\mathbf{X}]^2,
\end{aligned}
\label{eq:spectral_cues}
\end{equation}
where \(\mathbf{G}_x\) and \(\mathbf{G}_y = \mathbf{G}_x^\top\) are Sobel edge detectors, \(\mathbf{L}\) is the Laplacian curvature detector, and \(\mathbb{E}_{3\!\times\!3}[\cdot]\) denotes local averaging via a \(3\!\times\!3\) avg-pool (padding=1, stride=1) to estimate variance.
The resulting single-channel maps are concatenated to form the frequency descriptor \(\mathbf{F} \in \mathbb{R}^{3 \times H' \times W'}\).

A lightweight \(1\!\times\!1\) convolution with learnable weights \(\mathbf{W}_g\) and bias \(\mathbf{b}_g\) maps \(\mathbf{F}\) to per-pixel, per-scale logits:
\begin{equation}
    \mathbf{Z}_g = \mathbf{W}_g \ast \mathbf{F} + \mathbf{b}_g \in \mathbb{R}^{K \times H' \times W'},
\end{equation}
where \(K = |\mathcal{K}|\) is the number of peripheral scales.
A softmax over the scale dimension per pixel produces gating weights:
\begin{equation}
\alpha_k(h,w) = \frac{\exp(Z_{g,k}(h,w))}{\sum_{j=1}^{K} \exp(Z_{g,j}(h,w))}, \quad k \in \mathcal{K}.
\end{equation}
This ensures \(\sum_k \alpha_k(h,w) = 1\) at every pixel, enabling adaptive selection among scales.
Here, \(\boldsymbol{\alpha}_k \in \mathbb{R}^{H' \times W'}\) denotes the per-pixel gating map for scale \(k\).

\subsubsection{Multi-Scale Peripheral Response with Center Suppression}
For each scale \(k \in \mathcal{K} = \{9,15,31\}\), the large kernel is approximated separably to reduce computation:
\begin{equation}
\mathbf{P}_k = \mathbf{v}_k \ast (\mathbf{h}_k \ast \mathbf{X}) \in \mathbb{R}^{C' \times H' \times W'},
\end{equation}
where \(\mathbf{h}_k\) and \(\mathbf{v}_k\) are \(1\!\times\!k\) and \(k\!\times\!1\) depthwise convolution kernels, respectively.
A \(3\!\times\!3\) depthwise center kernel \(\mathbf{C}\) extracts local context.
The peripheral response with learnable center suppression is:
\begin{equation}
\mathbf{Y}_k = \mathbf{P}_k - \tanh(\boldsymbol{\beta}_k) \odot (\mathbf{C} \ast \mathbf{X}),
\end{equation}
where \(\boldsymbol{\beta}_k \in \mathbb{R}^{C'}\) is a learnable channel-wise parameter for scale \(k\) (broadcast over spatial dimensions). Each entry of \(\tanh(\boldsymbol{\beta}_k)\) lies in \([-1,1]\), allowing \emph{bidirectional modulation}. This sign-sensitive flexibility is crucial since feature maps contain both positive and negative activations, enabling the model to adaptively enhance or inhibit central context based on local spectral content.

The final output is a convex combination of the gated peripheral responses:
\begin{equation}
\mathrm{PFG}(\mathbf{X}) = \sum_{k \in \mathcal{K}} \boldsymbol{\alpha}_k \odot \mathbf{Y}_k.
\end{equation}
This forms adaptive spatial band-pass filters that emphasize mid-frequency motion patterns while suppressing redundant low-frequency background and high-frequency noise.

This formulation can be interpreted as a learnable center–surround operation: a large peripheral convolution captures broad contextual patterns, while a small central kernel suppresses redundant local responses. From a frequency-domain perspective, the kernel size determines the spectral selectivity:  a large spatial kernel corresponds to a \emph{narrow, low-pass} response that preserves slowly varying background information, whereas a small kernel produces a \emph{broader, higher-pass} response that is more sensitive to local edges and noise. Subtracting the two thus removes overlapping low-frequency components while retaining their difference in the middle band.

Formally, let \(H_{L_k}(\omega)\) and \(H_S(\omega)\) denote the frequency responses of the large and small kernels, both real-valued.  
Define
\(
H_k(\omega) = H_{L_k}(\omega) - \beta_k H_S(\omega).
\)
If there exist \(0 < r_1 < r_2 \le \pi\) such that
\(
H_{L_k}(\omega) > \beta_k H_S(\omega), \quad r_1 < \|\omega\| < r_2,
\)
and
\(
H_{L_k}(\omega) \le \beta_k H_S(\omega), \quad \text{otherwise},
\)
then \(H_k(\omega)\) behaves as a \textbf{ring-shaped band-pass filter} that amplifies mid-frequency components:
\begin{equation}
    H_k(\omega) > 0 \text{ for } \omega \in \Omega_m, \quad 
H_k(\omega) \le 0 \text{ otherwise,}
\end{equation}
where \(\Omega_m = \{\omega : r_1 < \|\omega\| < r_2\}\).  
This property explains how learnable center suppression enables large-kernel convolutions to selectively enhance motion- or edge-relevant frequencies while filtering out static low-frequency backgrounds and high-frequency noise. A more detailed theoretical analysis is provided in the supplementary material, where we show that, under mild conditions, \textbf{there exists a coefficient $\beta$ such that the composite filter $H_k(\omega)=H_{L_k}(\omega)-\beta_k H_S(\omega)$ achieves a higher signal-to-noise ratio than $H_{L_k}(\omega)$ alone.}

\subsubsection{Channel Mixing via GLU}
Given the preceding feature, we expand channels from \(C'\) to \(2E\) by a point-wise convolution and evenly split along channels into $\mathbf{U}$ and $\mathbf{V}$. A lightweight GLU~\cite{dauphin2017language} path then complements spatial adaptation:
\begin{equation}
\mathbf{Z}_c = \mathrm{PW}_{E \to C'}\big(\sigma(\mathbf{U}) \odot \mathrm{DW}_{3\!\times\!3}(\mathbf{V})\big),
\end{equation}
where \(E=4C'\) with expansion ratio \(r=4\); \(\mathrm{PW}\) denotes a point-wise \(1{\times}1\) convolution, and \(\mathrm{DW}_{3\times3}\) denotes a depthwise \(3{\times}3\) convolution. We also apply GRN normalization~\cite{woo2023convnext} and LayerScale~\cite{touvron2021going} for stable training.

\subsection{Output Decoding}
After \(N_t\) PFG blocks, the final latent tensor \(\mathbf{Z}' \in \mathbb{R}^{C' \times H' \times W'}\) is unpacked along the channel dimension to recover the temporal axis, yielding \(\{\mathbf{F}'_t \in \mathbb{R}^{C \times H' \times W'}\}_{t=1}^{T_{\mathrm{out}}}\). A symmetric decoder, following the SimVP design~\cite{gao2022simvp}, applies convolutional upsampling to restore each \(\mathbf{F}'_t\) to the original resolution, yielding output frames \(\mathbf{O}_t = \mathrm{Dec}(\mathbf{F}'_t)\). Collectively, these frames form the output sequence
\(\{\mathbf{O}_t \in \mathbb{R}^{C_{\mathrm{out}} \times H \times W}\}_{t=1}^{T_{\mathrm{out}}}\). The default setting is \(T_{\mathrm{out}} = T_{\mathrm{in}}\). In practice, if \(T_{\mathrm{out}} < T_{\mathrm{in}}\), the first \(T_{\mathrm{out}}\) frames are sliced; if \(T_{\mathrm{out}} > T_{\mathrm{in}}\), predictions are rolled out autoregressively by feeding the last output back as input.

\subsection{Computational Efficiency}
\label{sec:complexity}

Despite using large receptive fields (\(\mathcal{K} = \{9,15,31\}\)), PFGNet maintains high efficiency by \emph{decomposing every large 2D kernel into two 1D convolutions} (horizontal \(1 \times k\) followed by vertical \(k \times 1\)). This separable design reduces the per-channel complexity from \(\mathcal{O}(k^2)\) to \(\mathcal{O}(2k)\)—nearly linear in kernel size. For \(k=31\), this yields a 15× reduction in kernel parameters and multiply-adds per layer. 


\section{Experiments}
\label{sec:exp}
We evaluate \textbf{PFGNet} on four representative spatiotemporal prediction benchmarks: Moving MNIST~\cite{srivastava2015unsupervised}, TaxiBJ~\cite{zhang2017deep}, KTH~\cite{schuldt2004recognizing}, and Human3.6M~\cite{ionescu2013human3}. All datasets are obtained from the OpenSTL~\cite{tan2023openstl} repository and processed following its public pipeline. To ensure reproducibility and fair comparison, we align with the OpenSTL evaluation framework,  and we include influential recurrent and non-recurrent methods to establish a comprehensive benchmark.

\subsection{Experimental Setups}

\paragraph{Datasets}
\textbf{1. Moving MNIST}: The standard setting uses the first 10 frames for observation and predicts the next 10 frames. Both the training and test sets contain 10{,}000 sequences, serving to assess the stability and lower bound of temporal modeling. \textbf{2. TaxiBJ}:  The prevalent protocol adopts a temporal split with the last four weeks for testing and earlier data for training. All models observe 4 frames and predict 4 frames, highlighting the trade-off across numerical error, efficiency, and deployability.
\textbf{3. KTH}: The common split uses subjects 1 to 16 for training and 17 to 25 for testing. All models observe 10 frames, and predict 20 frames or 40 frames, focusing on long-range structural preservation and boundary consistency.  \textbf{4. Human3.6M}: We follow the OpenSTL public configuration and evaluation paradigm, with resolution set to $256\times256\times3$, and organize the training and testing pipeline accordingly to ensure consistency with community practice.

\paragraph{Metrics}
We adopt a complementary set of quality and efficiency metrics. {MSE} measures per-pixel mean squared error and lower values are better. {MAE} measures per-pixel mean absolute error and lower values are better. {SSIM} assesses image structural similarity in terms of structure, luminance, and contrast and higher values are better. {PSNR} measures reconstruction quality in a signal-to-noise ratio form and higher values are better. {Parameters} and {FLOPs} characterize model complexity and computational cost and smaller values are more deployment-friendly. 

\begin{table*}[htbp]
\centering
\caption{Experimental setup. \(T_{\mathrm{in}}\) and \(T_{\mathrm{out}}\) denote the numbers of input and predicted frames, respectively.}
\label{tab:setup}
{\setlength{\tabcolsep}{4.6pt} 
\begin{tabular}{c c c c c c c c c c c}
\toprule
Dataset & Resolution & \(T_{\mathrm{in}}\) & \(T_{\mathrm{out}}\) & $N_s$ & $N_t$ & DropPath & Epochs & Learning Rate & Batch Size & GPU \\
\midrule
Moving MNIST & $(64,64,1)$   & 10 & 10    & 4 & 8 & 0 & 2000 & $1\mathrm{e}{-3}$         & 16 & $1\times$ RTX 4090 \\
TaxiBJ       & $(32,32,2)$   & 4  & 4     & 2 & 8 & 0.1 & 50  & $2\mathrm{e}{-3}$          & 16 & $1\times$ RTX 3090 \\
KTH          & $(128,128,1)$ & 10 & 20/40 & 2 & 6 & 0.1 & 100  & $2\mathrm{e}{-4}/1\mathrm{e}{-4}$ & 2/1 & $4\times$ RTX 4090 \\
Human3.6M    & $(256,256,3)$ & 4  & 4     & 4 & 8 & 0.1 & 50   & $1.5\mathrm{e}{-3}$        & 2  & $5\times$ RTX 3090 \\
\bottomrule
\end{tabular}}
\end{table*}

\paragraph{Training Details}
To ensure fairness, we follow the OpenSTL training and evaluation procedures, including unified data preprocessing, evaluation protocols, and logging. Dataset-specific settings for spatial resolution, temporal windows, training epochs, learning rates, batch sizes, architectural hyperparameters, and hardware resources are summarized in Table~\ref{tab:setup}. All implementations are conducted and verified within a unified codebase.

\subsection{Main Results}

Tables~\ref{tab:mmnist}, \ref{tab:taxibj}, \ref{tab:kth}, and \ref{tab:human_dataset_comparison} report comprehensive quantitative comparisons between PFGNet and state-of-the-art baselines across four benchmarks. 
{Figures~\ref{fig:mmnist_qual}, \ref{fig:taxibj_qual}, \ref{fig:kth_qual} and \ref{fig:human_qual} showcase representative qualitative predictions from PFGNet.} 

\paragraph{Moving MNIST}
PFGNet achieves lower numerical error and higher structural similarity under the non-recurrent paradigm, showing stable gains over SimVP and TAU while still outperforming recurrent methods that emphasize temporal memory. The results indicate that frequency-guided gating highlights the most predictive mid-band components in short-term trajectory and occlusion scenarios, which reduces blur diffusion and ghosting while preserving clear digit boundaries and motion continuity.

\begin{table}[htbp]
\centering
\caption{Quantitative comparison on the Moving MNIST dataset.}
\begin{tabular}{c c c}
\toprule
Method & MSE $\downarrow$ & SSIM $\uparrow$ \\
\midrule
\multicolumn{3}{c}{Recurrent-based Methods} \\
\midrule
ConvLSTM~\cite{shi2015convolutional}   & 103.3 & 0.707 \\
DFN~\cite{jia2016dynamic}        & 89.0  & 0.726 \\
FRNN~\cite{oliu2018folded}       & 69.7  & 0.813 \\
VPN~\cite{kalchbrenner2017video}        & 64.1  & 0.870 \\
PredRNN~\cite{wang2017predrnn}    & 56.8  & 0.867 \\
CausalLSTM~\cite{wang2018predrnn++} & 46.5  & 0.898 \\
MIM~\cite{wang2019memory}        & 44.2  & 0.910 \\
E3D-LSTM~\cite{wang2018eidetic}   & 41.3  & 0.910 \\
LMC~\cite{lee2021video}        & 41.5  & 0.924 \\
MAU~\cite{chang2021mau}        & 27.6  & 0.937 \\
PhyDNet~\cite{guen2020disentangling}    & 24.4  & 0.947 \\
CrevNet~\cite{yu2020efficient}    & 22.3  & 0.949 \\
SwinLSTM~\cite{tang2023swinlstm}   & 17.7  & 0.962 \\
VMRNN~\cite{tang2024vmrnn}      & \underline{16.5}  & \underline{0.965} \\
\midrule
\multicolumn{3}{c}{Recurrent-free Methods} \\
\midrule
SimVP~\cite{gao2022simvp}      & 23.8  & 0.948 \\
MMVP~\cite{zhong2023mmvp}       & 22.2  & 0.952 \\
TAU~\cite{tan2023temporal}        & 19.8  & 0.957 \\
PFGNet(Ours) & \textbf{15.2} & \textbf{0.967} \\
\bottomrule
\end{tabular}
\label{tab:mmnist}
\end{table}

\begin{figure}[htbp]
\centering
\includegraphics[width=\linewidth]{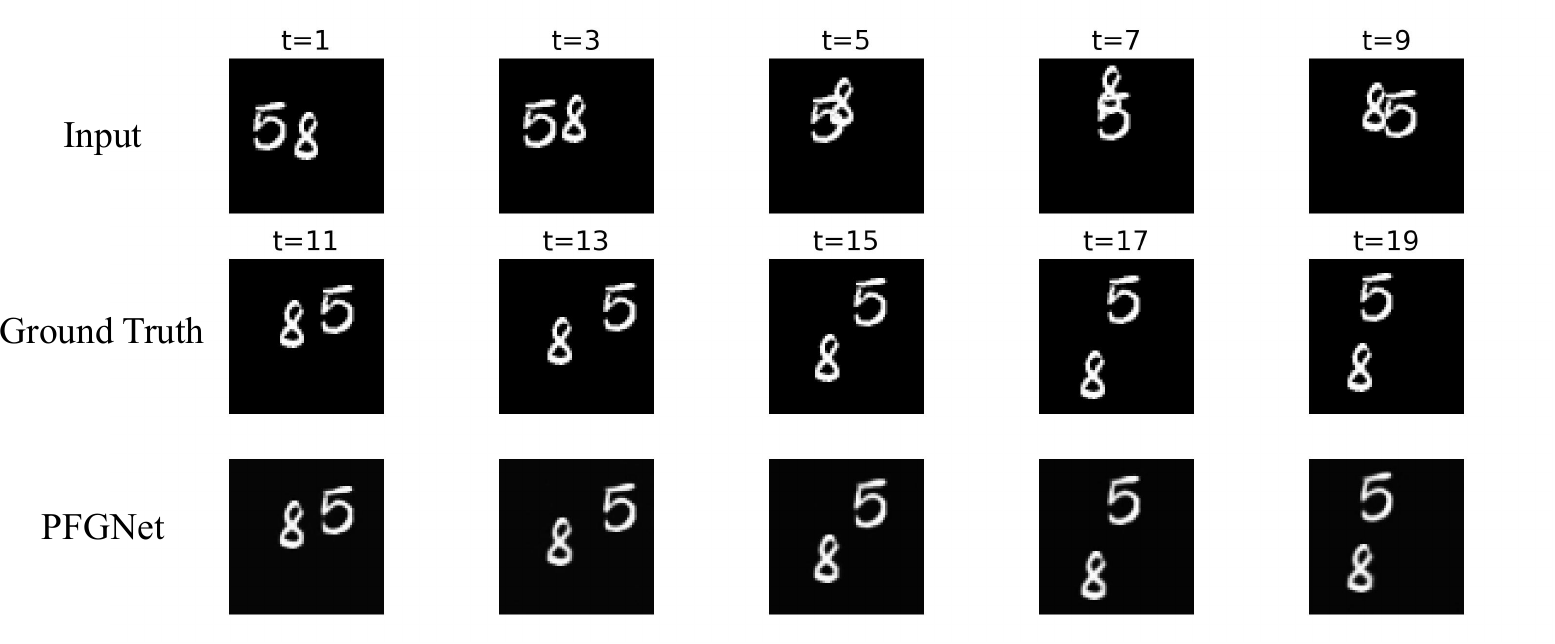}
\caption{
Qualitative results of PFGNet on Moving MNIST. 
}
\label{fig:mmnist_qual}
\vspace{-3mm}
\end{figure}

\paragraph{TaxiBJ}
On the TaxiBJ dataset (Table~\ref{tab:taxibj}), PFGNet achieves {leading MSE (0.2881)}, {near-best MAE (14.75)} and {SSIM (0.9857)} with only {1.9M parameters} and {0.6G FLOPs}—surpassing all recurrent-free baselines and recurrent methods. This demonstrates PFGNet's ability to stably extract {predictive structures}—robust, generalizable representations of inflow–outflow dynamics that persist across coupled spatiotemporal patterns. Compared with state-of-the-art recurrent models, PFGNet delivers {clear efficiency advantages} while maintaining {equal or superior accuracy}, making it ideal for compute-constrained and latency-sensitive deployments such as real-time urban traffic forecasting.

\begin{table}[htbp]
\centering
\setlength{\tabcolsep}{3.2pt}
\renewcommand{\arraystretch}{1.05}
\caption{Quantitative comparison on the TaxiBJ dataset.}
\resizebox{1.00\linewidth}{!}{%
\begin{tabular}{cccccc}
\toprule
Method & Params & FLOPs &
MSE $\downarrow$ & MAE $\downarrow$ & SSIM $\uparrow$ \\
\midrule
\multicolumn{6}{c}{Recurrent-based Methods} \\
\midrule
ConvLSTM~\cite{shi2015convolutional}   & 15.0M & 20.7G  & 0.3358 & 15.32 & 0.9836 \\
PredNet~\cite{lotter2017deep}          & 12.5M & 0.9G   & 0.3516 & 15.91 & 0.9828 \\
PredRNN~\cite{wang2017predrnn}         & 23.7M & 42.4G  & 0.3194 & 15.31 & 0.9838 \\
PredRNN++~\cite{wang2018predrnn++}     & 38.4M & 63.0G  & 0.3348 & 15.37 & 0.9834 \\
E3D-LSTM~\cite{wang2018eidetic}        & 51.0M & 98.2G  & 0.3421 & 14.98 & 0.9842 \\
PhyDNet~\cite{guen2020disentangling}   & 3.1M  & 5.6G   & 0.3622 & 15.53 & 0.9828 \\
MIM~\cite{wang2019memory}              & 37.9M & 64.1G  & 0.3110 & 14.96 & 0.9847 \\
MAU~\cite{chang2021mau}                & 4.4M  & 6.4G   & 0.3062 & 15.26 & 0.9840 \\
PredRNNv2~\cite{wang2022predrnn}       & 23.7M & 42.6G  & 0.3834 & 15.55 & 0.9826 \\
SwinLSTM~\cite{tang2023swinlstm}       & 2.9M  & 1.3G   & 0.3026 & 15.00 & 0.9843 \\
VMRNN~\cite{tang2024vmrnn}             & \underline{2.6M}  & \underline{0.9G}   & 0.2887 & \textbf{14.69} & \textbf{0.9858} \\
\midrule
\multicolumn{6}{c}{Recurrent-free Methods} \\
\midrule
SimVP~\cite{gao2022simvp}              & 13.8M & 3.6G   & 0.3282 & 15.45 & 0.9835 \\
TAU~\cite{tan2023temporal}             & 9.6M  & 2.5G   & 0.3108 & 14.93 & 0.9848 \\
SimVPv2~\cite{tan2211simvp}            & 10.0M & 2.6G   & 0.3246 & 15.03 & 0.9844 \\
Met2Net~\cite{li2025met2net}           & 16.8M & 8.8G   & 0.3164 & 14.82 & 0.9851 \\
PFGNet(Ours)                             & \textbf{1.9M}  & \textbf{0.6G} & \textbf{0.2881} & \underline{14.75} & \underline{0.9857} \\
\bottomrule
\end{tabular}
}
\label{tab:taxibj}
\end{table}

\begin{figure}[htbp]
\centering
\includegraphics[width=\linewidth]{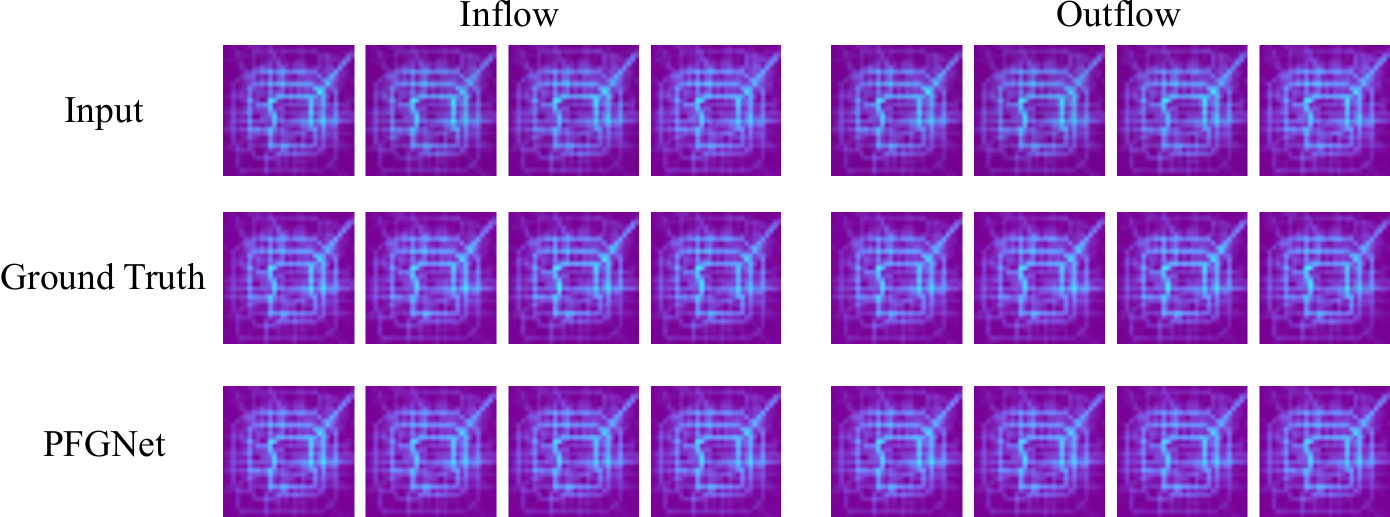}
\caption{
Qualitative results of PFGNet on TaxiBJ. 
}
\label{fig:taxibj_qual}
\vspace{-3mm}
\end{figure}

\paragraph{KTH}
On the KTH dataset (Table~\ref{tab:kth}), PFGNet achieves the {highest SSIM} across all methods, with competitive PSNR (34.10 and 32.64). This SSIM dominance reflects PFGNet’s superior preservation of limb contours, joint trajectories, and motion topology over long horizons—crucial for articulated action forecasting. The {slightly lower PSNR} arises because KTH features high-contrast, near-static backgrounds that dominate pixel-wise error. PSNR penalizes minor intensity shifts in large uniform regions, whereas PFGNet intentionally de-emphasizes photometric fidelity to prevent limb collapse and boundary diffusion—yielding anatomically coherent, perceptually superior predictions.

\begin{table}[htbp]
\centering
\setlength{\tabcolsep}{3.5pt}
\renewcommand{\arraystretch}{1.05}
\caption{Quantitative comparison on the KTH dataset.}
\begin{tabular}{ccccc} 
\toprule
\multirow{2}{*}{Method} &
\multicolumn{2}{c}{KTH (10 $\rightarrow$ 20)} &
\multicolumn{2}{c}{KTH (10 $\rightarrow$ 40)} \\
\cmidrule(lr){2-3} \cmidrule(lr){4-5}
& SSIM $\uparrow$ & PSNR $\uparrow$ & SSIM $\uparrow$ & PSNR $\uparrow$ \\
\midrule
\multicolumn{5}{c}{Recurrent-based Methods} \\
\midrule
ConvLSTM~\cite{shi2015convolutional}   & 0.712 & 23.58 & 0.639 & 22.85 \\
SAVP~\cite{lee2018stochastic}          & 0.746 & 25.38 & 0.701 & 23.97 \\
FRNN~\cite{oliu2018folded}             & 0.771 & 26.12 & 0.678 & 23.77 \\
DFN~\cite{jia2016dynamic}              & 0.794 & 27.26 & 0.652 & 23.01 \\
PredRNN~\cite{wang2017predrnn}         & 0.839 & 27.55 & 0.703 & 24.16 \\
VarNet~\cite{jin2018varnet}            & 0.843 & 28.48 & 0.739 & 25.37 \\
SAVP-VAE~\cite{lee2018stochastic}      & 0.852 & 27.77 & 0.811 & 26.18 \\
PredRNN++~\cite{wang2018predrnn++}     & 0.865 & 28.47 & 0.741 & 25.21 \\
E3D-LSTM~\cite{wang2018eidetic}        & 0.879 & 29.31 & 0.810 & 27.24 \\
STMFANet~\cite{jin2020exploring}       & 0.893 & 29.85 & 0.851 & 27.56 \\
SwinLSTM~\cite{tang2023swinlstm}       & 0.903 & \textbf{34.34} & 0.879 & \textbf{33.15} \\
VMRNN~\cite{tang2024vmrnn}             & \underline{0.907} & 34.06 & 0.882 & 32.69 \\
\midrule
\multicolumn{5}{c}{Recurrent-free Methods} \\
\midrule
SimVP~\cite{gao2022simvp}              & 0.905 & 33.72 & 0.886 & \underline{32.93} \\
MMVP~\cite{zhong2023mmvp}              & 0.906 & 27.54 & \underline{0.888} & 26.35 \\
PFGNet(Ours)                              & \textbf{0.911} & \underline{34.10} & \textbf{0.891} & 32.64 \\
\bottomrule
\end{tabular}
\label{tab:kth}
\end{table}

\begin{figure}[htbp]
\centering
\includegraphics[width=\linewidth]{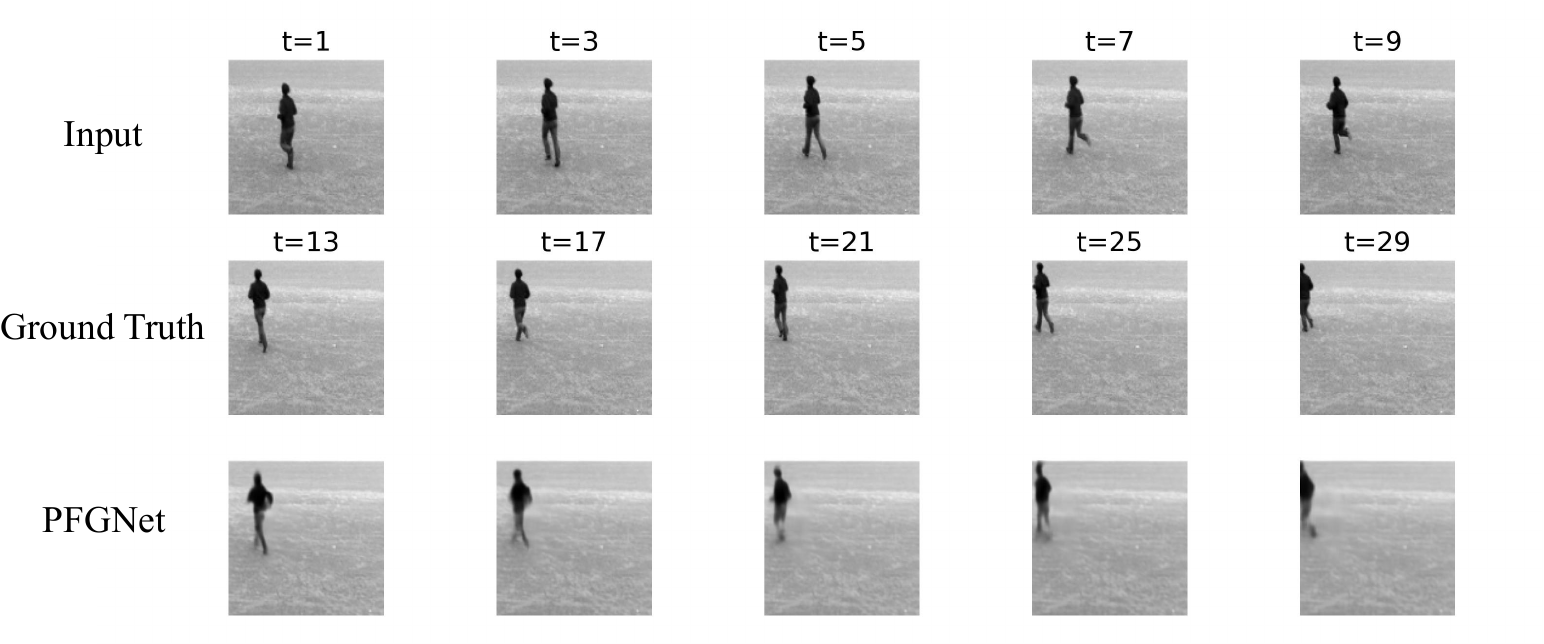}
\caption{
Qualitative results of PFGNet on KTH. 
}
\label{fig:kth_qual}
\vspace{-3mm}
\end{figure}

\begin{table}[htbp]
\footnotesize
\centering
\setlength{\tabcolsep}{3.5pt}
\caption{Quantitative comparison on the Human3.6M dataset.}
\begin{tabular}{c c c c c c}
\toprule
Method & Params & FLOPs & MSE $\downarrow$ & MAE $\downarrow$ & SSIM $\uparrow$ \\
\midrule
\multicolumn{6}{c}{Recurrent-based Methods} \\
\midrule
ConvLSTM~\cite{shi2015convolutional}& 15.5M & 347.0G & 125.5 & 1566.7 & 0.9813 \\
E3D-LSTM~\cite{wang2018eidetic} & 60.9M & 542.0G & 143.3 & 1442.5 & 0.9803 \\
PredNet~\cite{lotter2017deep} & 12.5M & \textbf{13.7G}  & 261.9 & 1625.3 & 0.9786 \\
PhyDNet~\cite{guen2020disentangling} & \underline{4.2M}  & \underline{19.1G}  & 125.7 & 1614.7 & 0.9804 \\
MAU~\cite{chang2021mau} & 20.2M & 105.0G & 127.3 & 1577.0 & 0.9812 \\
MIM~\cite{wang2019memory} & 47.6M & 1051.0G & 112.1 & 1467.1 & 0.9829 \\
PredRNN~\cite{wang2017predrnn}  & 24.6M & 704.0G & 113.2 & 1458.3 & 0.9831 \\
PredRNN++~\cite{wang2018predrnn++} & 39.3M & 1033.0G & 110.0 & 1452.2 & 0.9832 \\
PredRNNv2~\cite{wang2022predrnn} & 24.6M & 708.0G & 114.9 & 1484.7 & 0.9827 \\
\midrule
\multicolumn{6}{c}{Recurrent-free Methods} \\
\midrule
SimVP~\cite{gao2022simvp}     & 41.2M & 197.0G & 115.8 & 1511.5 & 0.9822 \\
SimVPv2~\cite{tan2211simvp}   & 11.3M & 74.6G  & \textbf{108.4} & 1441.0 & 0.9834 \\
TAU~\cite{tan2023temporal}    & 37.6M & 182.0G & 113.3 & \textbf{1390.7} & \textbf{0.9839} \\
ViT~\cite{dosovitskiy2021an}  & 28.3M & 239.0G & 136.3 & 1603.5 & 0.9796 \\
Swin Transformer~\cite{liu2021swin}& 38.8M & 188.0G & 133.2 & 1599.7 & 0.9799 \\
Uniformer~\cite{li2022uniformer}& 27.7M & 211.0G & 116.3 & 1497.7 & 0.9824 \\
MLP-Mixer~\cite{tolstikhin2021mlp}& 47.0M & 164.0G & 125.7 & 1511.9 & 0.9819 \\
ConvMixer~\cite{trockman2023patches}& \textbf{3.1M}  & 39.4G  & 115.8 & 1527.4 & 0.9822 \\
Poolformer~\cite{yu2022metaformer}& 31.2M & 156.0G & 118.4 & 1484.1 & 0.9827 \\
ConvNeXt~\cite{liu2022convnet} & 31.4M & 157.0G & 113.4 & 1469.7 & 0.9828 \\
VAN~\cite{guo2023visual} & 37.5M & 182.0G & 111.4 & 1454.5 & 0.9831 \\
HorNet~\cite{rao2022hornet} & 28.1M & 143.0G & 118.1 & 1481.1 & 0.9824 \\
MogaNet~\cite{moga}           & 8.6M  & 63.6G  & \underline{109.1} & 1446.4 & 0.9834 \\
PFGNet(Ours)        & 7.3M & 58.3G & 111.3 & \underline{1392.4} & \underline{0.9838} \\
\bottomrule
\end{tabular}
\label{tab:human_dataset_comparison}
\end{table}

\begin{figure}[htbp]
\centering
\includegraphics[width=0.9\linewidth]{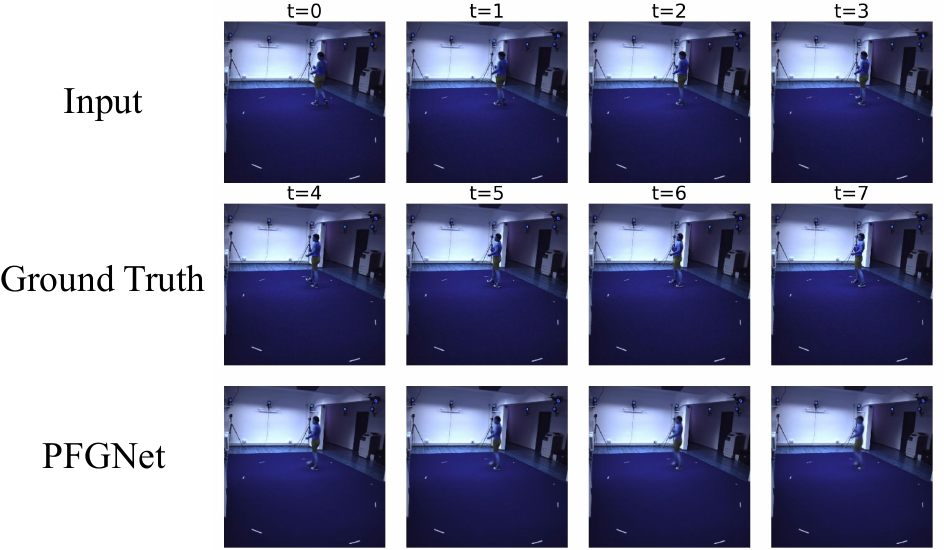}
\caption{
Qualitative results of PFGNet on Human3.6M. 
}
\label{fig:human_qual}
\vspace{-3mm}
\end{figure}

\paragraph{Human3.6M}
On the high-resolution Human3.6M dataset (256$\times$256, rich texture, complex articulated motion), PFGNet achieves a {near-optimal performance-complexity trade-off}: with only {7.3M parameters} and {58.3G FLOPs}, it attains {MAE = 1392.4} (2nd best) and {SSIM = 0.9838} (2nd best)—surpassing most recurrent baselines and all recurrent-free models except TAU (which uses 5.2$\times$ more parameters and 3.1$\times$ more FLOPs). In contrast, recurrent models (e.g., PredRNN++, MIM) require 3–6$\times$ more parameters and 10–18$\times$ more FLOPs to achieve comparable SSIM, while transformer-based recurrent-free designs sacrifice throughput for marginal error reduction. PFGNet thus {breaks the accuracy–efficiency barrier} in full-convolutional STPL, delivering {deployment-ready performance} without compromising long-term motion coherence or human detail preservation.

\subsection{Ablation Study}
To quantify the true contribution of each component, we conduct various groups of ablation experiments. The results are given in Tables~\ref{tab:ablation1} and \ref{tab:ablation2} and Figure~\ref{fig:ab_fig}. 

We first ablate the macro-structural design of PFGNet (Table~\ref{tab:ablation1} and Figure~\ref{fig:ab1}).
MSInit serves as a lightweight multi-scale initializer: removing it (model1) significantly degrades all metrics, confirming its role in providing diverse frequency-aware features for subsequent gating.
For multi-scale fusion, the softmax gating mechanism—through its ability to perform pixel-wise adaptive weighting—consistently surpasses fixed-weight strategies (model2). Overall, under the current setting, the \(3{\times}3\) center suppression kernel shows advantages over the \(5{\times}5\) one; nevertheless, the kernel size should generally remain aligned with the scale of the main convolution.
Figure~\ref{fig:ab1} shows that single-scale kernels, regardless of their size, consistently underperform the combination of multiple kernels. Interestingly, the largest kernel (\(k{=}31\)) achieves the {lowest MSE} yet the {highest SSIM} when used alone---providing strong empirical evidence for our hypothesis that {extremely large receptive fields are essential for STPL}. However, no single scale offers balanced spectral coverage: smaller kernels capture fine textures but miss global coherence, while larger ones excel in structure preservation but lose local detail. The {parallel fusion} of multiple scales achieves the {best performance}, confirming the necessity of multi-scale integration.

\begin{table}[t]
\small
\centering
\caption{Ablation on macro structure. We analyze the roles of MSInit and different fusion strategies under a unified kernel setting ($\mathcal{K}=\{9,15,31\}$).}
\setlength{\tabcolsep}{3.3pt}
\renewcommand{\arraystretch}{1.15}
\begin{tabular}{cc|c|c|c|c}
\hline
 &  & model1 & model2 & model3 & PFGNet \\
\hline
\multirow{1}{*}{\centering Archi}
& MSInit             &              & \checkmark & \checkmark & \checkmark \\
\hline
\multirow{2}{*}{\centering Fusion}
& mean               &              & \checkmark &            &  \\
& softmax            & \checkmark   &            & \checkmark & \checkmark \\
\hline
\multirow{2}{*}{\centering Center}
& $3{\times}3$       & \checkmark   & \checkmark &            & \checkmark \\
& $5{\times}5$       &              &            & \checkmark &            \\
\hline
\multirow{3}{*}{\centering TaxiBJ}
& MSE $\downarrow$   & 0.3119 & 0.3033 & 0.2995 & \textbf{0.2881} \\
& MAE $\downarrow$   & 14.89  & 14.89  & 14.90  & \textbf{14.75} \\
& SSIM $\uparrow$    & 0.9849 & 0.9849 & 0.9848 & \textbf{0.9857} \\
\hline
\multirow{2}{*}{\centering KTH 10$\rightarrow$20}
& SSIM $\uparrow$    & 0.909 &  0.909   & 0.904    & \textbf{0.911}       \\
& PSNR $\uparrow$    & 32.88 &  33.08   & 33.61    & \textbf{34.10}\\
\hline
\end{tabular}
\label{tab:ablation1}
\end{table}

\begin{figure}[t]
    \centering
    \begin{subfigure}[t]{0.49\linewidth}
        \centering
        \includegraphics[width=\linewidth, height=0.7\linewidth, keepaspectratio=false]{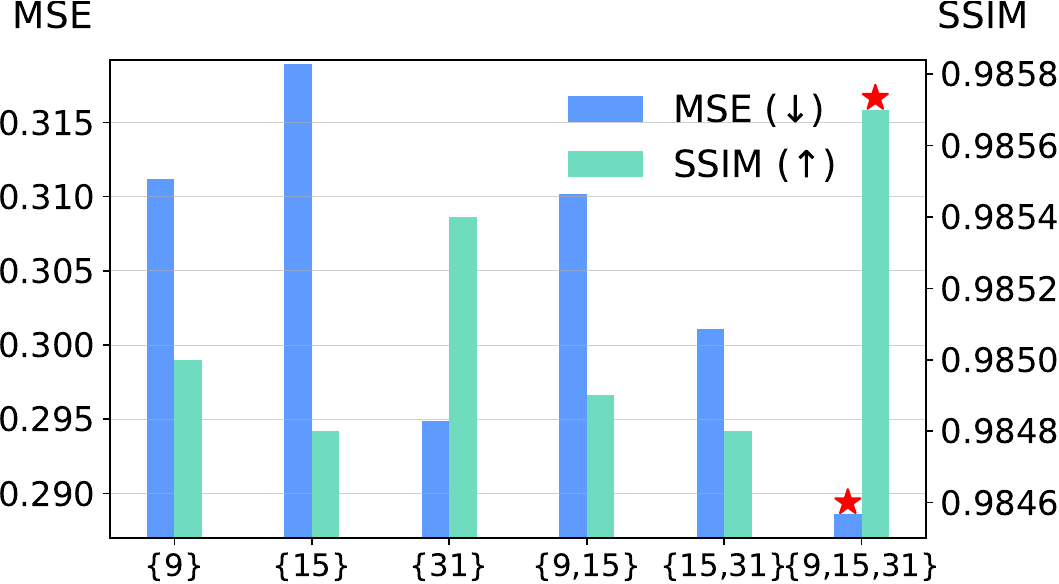}
        \caption{Ablation on kernel-set combinations in PFG block.}
        \label{fig:ab1}
    \end{subfigure}
    \hfill
    \begin{subfigure}[t]{0.49\linewidth}
        \centering
        \includegraphics[width=\linewidth, height=0.7\linewidth, keepaspectratio=false]{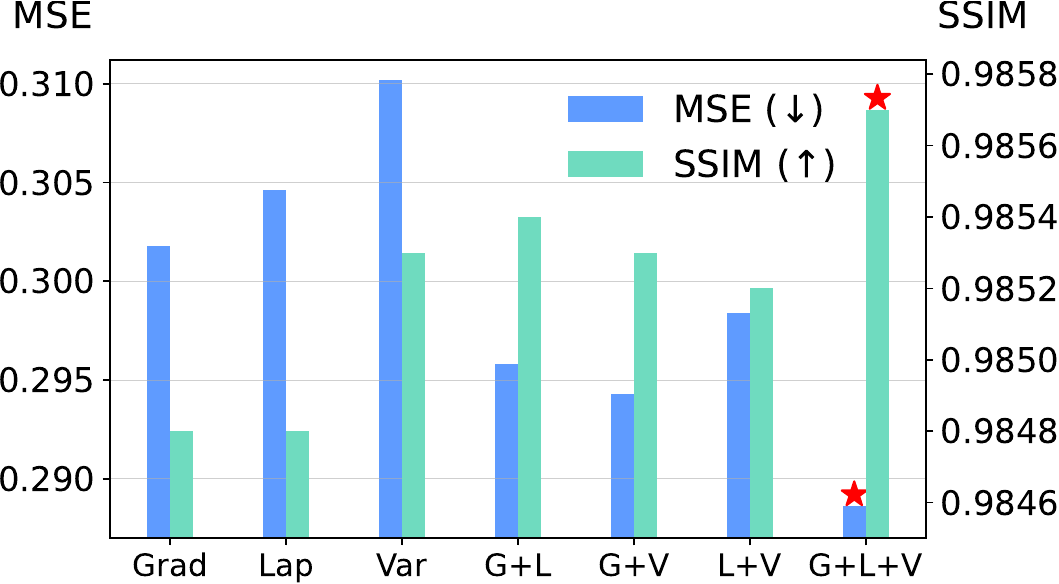}
        \caption{Ablation on frequency-cue combinations.}
        \label{fig:ab2}
    \end{subfigure}
    \vspace{-1mm}
    \caption{Ablation results on TaxiBJ.}
    \label{fig:ab_fig}
\end{figure}

\begin{table}[t]
\small
\centering
\caption{Ablation on detail mechanisms under full configuration (MSInit + softmax + all operators). We compare fixed vs.\ learnable $\beta$, and activation functions used for the suppression gate.}
\setlength{\tabcolsep}{1.5pt}
\renewcommand{\arraystretch}{1.15}
\begin{tabular}{cc|ccc|c|c}
\hline
 &  & model1 & model2 & model3 & model4 & PFGNet \\
\hline
\multirow{4}{*}{\centering $\beta$}
& $-1$                & \checkmark &            &            &            &            \\
& $1$                 &            & \checkmark &            &            &            \\
& $0$                 &            &            & \checkmark &            &            \\
& learnable           &            &            &            & \checkmark & \checkmark \\
\hline
\multirow{2}{*}{\centering Gate}
& tanh                & \checkmark & \checkmark & \checkmark &            & \checkmark \\
& sigmoid             &            &            &            & \checkmark &            \\
\hline
\multirow{3}{*}{\centering TaxiBJ}
& MSE $\downarrow$    & 0.3209 & 0.3286 & 0.2993 & 0.3142 & \textbf{0.2881} \\
& MAE $\downarrow$    & 14.96  & 15.02  & 14.86  & 15.00  & \textbf{14.75}  \\
& SSIM $\uparrow$     & 0.9849 & 0.9844 & 0.9850 & 0.9845 & \textbf{0.9857} \\
\hline
\multirow{2}{*}{\centering\shortstack{Moving MNIST\\(100 epochs)}}
& MSE $\downarrow$    & 28.06   & 28.14 & 27.84   & 28.04  & \textbf{27.61} \\
& SSIM $\uparrow$     & 0.9371  & 0.9369 & 0.9377 & 0.9372 & \textbf{0.9381} \\
\hline
\end{tabular}
\label{tab:ablation2}
\end{table}

We further analyze the internal mechanisms of the PFG block (Table~\ref{tab:ablation2} and Figure~\ref{fig:ab2}).  
The learnable center suppression parameter ($\beta$) significantly outperforms fixed values: $\beta=0$ (no suppression) retains redundant low-frequency background (model3), while fixed $\beta=\pm1$ (models 1--2) lacks spatial adaptivity.  
Using \textit{tanh} as the gating function for $\beta$ yields better results than \textit{sigmoid}, indicating that both positive and negative feedback contribute to feature modulation.  
The three frequency cues are complementary—removing any single cue or any pair leads to performance degradation, whereas using all three in parallel provides the most stable gating signal for pixel-wise scale selection.  

The optimal configuration—\textbf{MSInit + softmax fusion + $\mathcal{K}=\{9,15,31\}$ + learnable $\beta$ with tanh gating + all three frequency cues}—achieves the best results, validating the synergy between \emph{frequency-guided adaptive band-pass filtering} and \emph{separable large-kernel efficiency}, which enables PFGNet to outperform recurrent, attention-based, and small-kernel baselines with minimal computational overhead.

\section{Conclusion}
\label{sec:con}
In this work, we present PFGNet, a fully convolutional framework for spatiotemporal predictive learning that mimics biological center–surround processing via pixel-wise frequency-guided gating and learnable center suppression. Under the standardized OpenSTL protocol, PFGNet delivers SOTA or near-SOTA forecasting accuracy with substantially fewer parameters and FLOPs than recent recurrent, attention-based, or hybrid models. Lightweight, fully parallelizable, and plug-and-play, PFGNet provides a scalable real-time solution for high-precision predictive vision and is well suited for integration into larger backbones and deployment in applications such as autonomous systems and weather nowcasting.
\section*{Acknowledgments}
This work was supported in part by the National Natural Science Foundation of China under Grant No.~62406206 and by the Fundamental Research Funds for the Central Universities. We also extend our sincere appreciation to the providers of the datasets used in this study, whose contributions were essential to the experimental validation of the proposed method.
{
    \small
    \bibliographystyle{ieeenat_fullname}
    \bibliography{main}
}
\clearpage
\setcounter{page}{1}
\maketitlesupplementary

\section{Existence and Optimality of Ring-Shaped Pass Band}

In the main text, we established a \textbf{strong existence} result under monotonicity,
guaranteeing a ring-shaped pass band for $\beta_k \in (0, \beta_{\max})$.
Here we present a \textbf{weak existence} theorem that aligns with PFGNet's
implementation, where $\beta = \tanh(\beta_{\text{raw}}) \in (-1, 1)$.

\begin{theorem}[Weak Existence of Ring-Shaped Pass Band]
Let \(H_1, H_2 : [0, \pi] \to \mathbb{R}\) denote continuous radial frequency-response profiles.
Define $f(r) = H_1(r) - \beta H_2(r)$ for $\beta \in (-1, 1)$.
Assume there exist $0 \le c < a < b \le \pi$ such that
$$ f(c) \le 0,\quad f(a) > 0,\quad f(b) \le 0. $$
Then there exist $r_1$ and $r_2$ with $c \le r_1 < a < r_2 \le b$ such that
$$ f(r) > 0 \quad \text{for all } r \in (r_1, r_2), \text{ and } f(r_1) = f(r_2) = 0. $$
Thus $H_\beta(\omega) = H_1(\|\omega\|) - \beta H_2(\|\omega\|)$ has a non-degenerate
ring-shaped pass band $\{\omega : r_1 < \|\omega\| < r_2\}$.
\end{theorem}

\begin{proof}
Since $f$ is continuous with $f(c) \le 0$ and $f(a) > 0$, the set $S_1 = \{r \in [c, a] : f(r) \le 0\}$ is non-empty and closed. We define $r_1 = \max S_1$. Clearly, $c \le r_1 < a$ and $f(r_1) = 0$. By the definition of the maximum, $f(r) > 0$ for all $r \in (r_1, a]$.

Similarly, the set $S_2 = \{r \in [a, b] : f(r) \le 0\}$ is non-empty and closed because $f(a) > 0$ and $f(b) \le 0$. We define $r_2 = \min S_2$. Thus $a < r_2 \le b$ and $f(r_2) = 0$. By the definition of the minimum, $f(r) > 0$ for all $r \in [a, r_2)$.

Combining these results, we have $f(r) > 0$ on the continuous interval $(r_1, r_2)$. 
\end{proof}

In PFGNet, we set:
\begin{itemize}
    \item $H_1$: frequency response of a large-kernel convolution,
    \item $H_2$: frequency response of a small-kernel convolution,
    \item $\beta = \tanh(\beta_{\text{raw}}) \in (-1, 1)$: bounded center-suppression coefficient.
\end{itemize}
The large kernel decays slowly from DC to mid-frequencies, while the small kernel
decays rapidly. Their difference $f(r) = H_1(r) - \beta H_2(r)$ naturally satisfies:
\begin{itemize}
    \item $f(c) \le 0$ for small $c > 0$ (small kernel dominates near DC),
    \item $f(a) > 0$ in mid-frequencies (large kernel retains more energy),
    \item $f(b) \le 0$ near $b \approx \pi$ (both decay to zero).
\end{itemize}
Hence, the combination of large- and small-kernel convolutions ensures the
sign pattern required by the corrected weak theorem, yielding a non-degenerate
ring-shaped pass band in practice. We illustrate the frequency-selective behavior of the PFG block in Figure~\ref{fig:pfg-ring-filter}. 
Using a large kernel with a slowly decaying frequency response $H_{L_k}$ (e.g., exponential decay with rate 0.6) and a small kernel with a rapidly decaying response $H_S$ (e.g., Gaussian decay with variance parameter 2.5), their difference $H_k = H_{L_k} - \beta_k H_S$ (with $\beta_k = 0.75$) suppresses low-frequency background near DC, amplifies mid-frequency motion cues, and attenuates high-frequency noise, forming a ring-shaped band-pass filter as predicted by our theoretical analysis.

\begin{figure}[t]
    \centering
    \includegraphics[width=\linewidth]{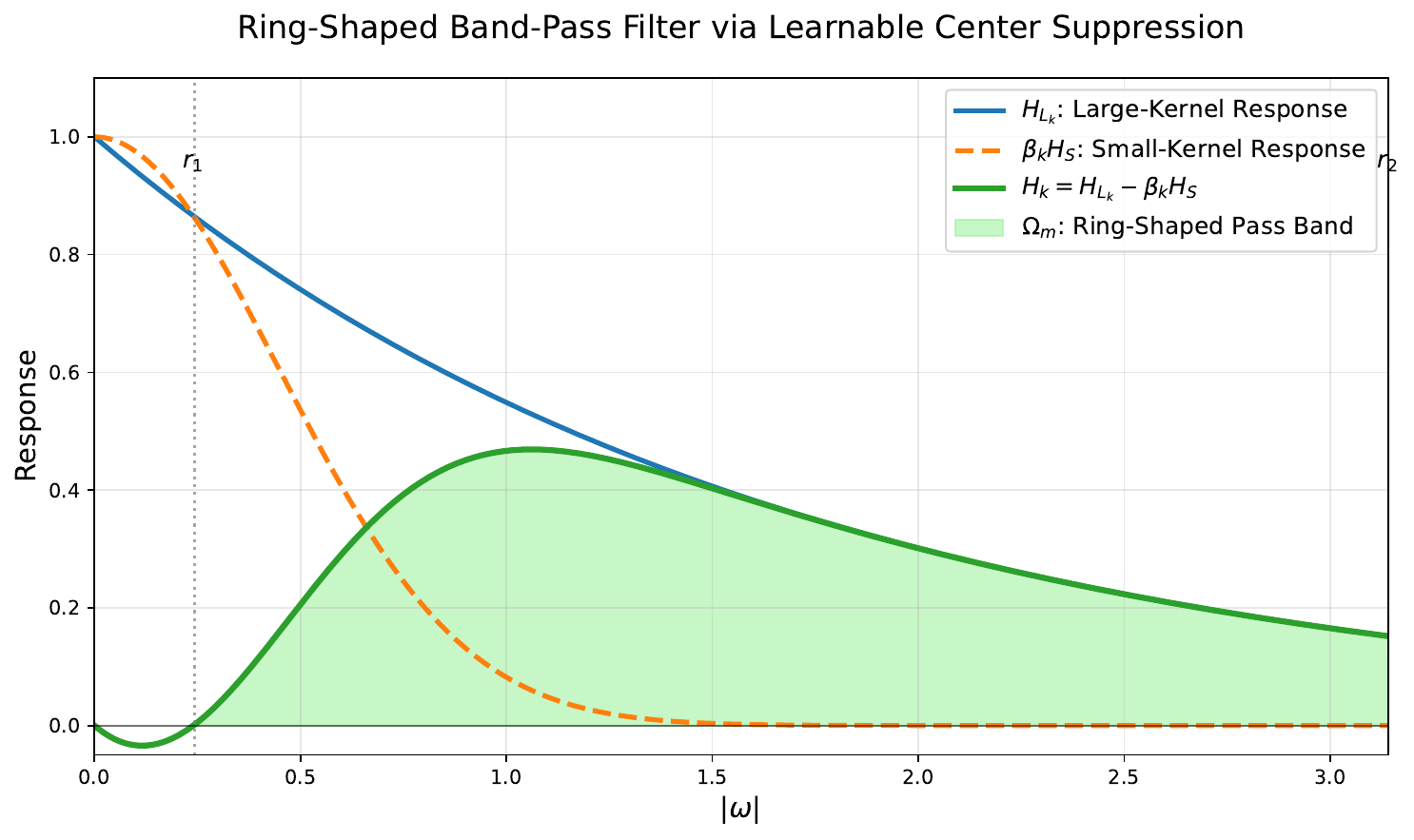}
    \caption{
        Frequency response of the PFG block. The large-kernel response $H_{L_k}$ decays slowly, while the small-kernel response $H_S$ decays rapidly. Their difference $H_k = H_{L_k} - \beta_k H_S$ (with $\beta_k{=}0.75$) is positive within the green region, forming a ring-shaped band-pass filter as stated in the main text.
    }
    \label{fig:pfg-ring-filter}
    \vspace{-4mm} 
\end{figure}

Furthermore, we have the following theorem:
\begin{theorem}[Existence of an SNR-maximizing $\beta^\star$]
Assume the input signal has spectral power $P_S(\omega) \ge 0$ and the additive noise
is white with constant power $P_N(\omega)\equiv\sigma_N^2>0$.
Let the composite filter response be $H_\beta = H_L - \beta H_S$, where
$H_L$ and $H_S$ denote the frequency responses of the large and small
kernels, respectively.
Define the signal-to-noise ratio
$$ \mathrm{SNR}(\beta) = \frac{\int |H_\beta(\omega)|^2 P_S(\omega)\,d\omega}{\int |H_\beta(\omega)|^2 P_N(\omega)\,d\omega}. $$
If $H_L$ and $H_S$ are linearly independent in $L^2([0,\pi])$ and $\int |H_S|^2 P_S\,d\omega > 0$, then $\mathrm{SNR}(\beta)$ admits at least one finite stationary point $\beta^\star$ satisfying $\tfrac{d}{d\beta}\mathrm{SNR}(\beta)=0$.
\label{theorem:2}
\end{theorem}

\begin{proof}
Define the signal energy $N(\beta)$ and noise energy $D(\beta)$ as quadratic forms:
$$ N(\beta) := \int |H_L - \beta H_S|^2 P_S(\omega)\, d\omega = A - 2\beta B + \beta^2 C, $$
{\footnotesize$$ D(\beta) := \int |H_L - \beta H_S|^2 P_N(\omega)\, d\omega = \sigma_N^2 \left( \tilde{A} - 2\beta \tilde{B} + \beta^2 \tilde{C} \right), $$}
where the coefficients are given by the respective integrals (e.g., $C = \int |H_S|^2 P_S\, d\omega$ and $\tilde{C} = \int |H_S|^2\, d\omega$).

By the linear independence of $H_L$ and $H_S$ in $L^2$, the denominator $D(\beta)$ is strictly positive for all $\beta \in \mathbb{R}$, ensuring that $\mathrm{SNR}(\beta) = \frac{N(\beta)}{D(\beta)}$ is continuously differentiable on the entire real line.

Consider the asymptotic behavior of $\mathrm{SNR}(\beta)$ as $\beta \to \pm\infty$. Dividing the numerator and denominator by $\beta^2$ yields:
$$ \lim_{\beta \to \pm\infty} \mathrm{SNR}(\beta) = \lim_{\beta \to \pm\infty} \frac{A/\beta^2 - 2B/\beta + C}{\sigma_N^2 (\tilde{A}/\beta^2 - 2\tilde{B}/\beta + \tilde{C})} = \frac{C}{\sigma_N^2 \tilde{C}}. $$
By assumption, $C > 0$ and $\tilde{C} > 0$, so this limit is a finite positive constant, which we denote as $L$.

We now examine the continuous function $\mathrm{SNR}(\beta)$. There are three possible cases:
1. $\mathrm{SNR}(\beta) \equiv L$ for all $\beta$. Then every $\beta \in \mathbb{R}$ is a stationary point, and the theorem holds trivially.
2. There exists some $\beta_0 \in \mathbb{R}$ such that $\mathrm{SNR}(\beta_0) > L$. Since the function is continuous and decays to $L$ at both infinities, by the Extreme Value Theorem, $\mathrm{SNR}(\beta)$ must attain a global maximum at some finite point $\beta^\star$.
3. There exists some $\beta_0 \in \mathbb{R}$ such that $\mathrm{SNR}(\beta_0) < L$. By the same reasoning, $\mathrm{SNR}(\beta)$ attains a global minimum at some finite point $\beta^\star$.

In cases 2 and 3, Fermat's theorem guarantees that at the local extremum $\beta^\star$, the derivative must vanish: $\frac{d}{d\beta}\mathrm{SNR}(\beta^\star) = 0$.

Furthermore, setting the derivative to zero yields:
$$ N'(\beta)D(\beta) - N(\beta)D'(\beta) = 0. $$
Expanding this expression, the $\beta^3$ terms mathematically cancel out ($C\tilde{C}\beta^3 - \tilde{C}C\beta^3 = 0$), leaving a polynomial equation of degree at most 2:
$$ (B\tilde{C} - C\tilde{B})\beta^2 + (C\tilde{A} - A\tilde{C})\beta + (\tilde{B}A - B\tilde{A}) = 0. $$
This confirms that $\mathrm{SNR}(\beta)$ has at most two finite stationary points, and we have proven that at least one must exist.
\end{proof}

\begin{lemma}[SNR Advantage of $H_L - \beta H_S$ over $H_L$]
Assume the input signal has spectral power $P_S(\omega) \ge 0$ and additive white noise with power $\sigma_N^2 > 0$. 
Let $H_L$ and $H_S$ be the frequency responses of the large and small kernels, respectively, 
and assume they are linearly independent in $L^2([0,\pi])$. 
Define the composite filter $H_\beta = H_L - \beta H_S$ and the SNR as
$$ \mathrm{SNR}(\beta) = \frac{\int |H_\beta|^2 P_S\, d\omega}{\int |H_\beta|^2 \sigma_N^2\, d\omega}. $$
Assume further that $H_L$ is not already a stationary point of the SNR in the direction of $H_S$, meaning
{\footnotesize$$ \int \mathrm{Re}(H_L \overline{H_S}) P_S\, d\omega \int |H_L|^2\, d\omega \neq \int |H_L|^2 P_S\, d\omega \int \mathrm{Re}(H_L \overline{H_S})\, d\omega $$}
(i.e., $B\tilde{A} \neq A\tilde{B}$ in the notation below).
Then, there exists $\hat{\beta} \neq 0$ such that
$$ \mathrm{SNR}(\hat{\beta}) > \mathrm{SNR}(0). $$
\end{lemma}

\begin{proof}
Define the energy integrals
$$ N(\beta) = A - 2\beta B + \beta^2 C, \qquad D(\beta) = \sigma_N^2(\tilde{A} - 2\beta\tilde{B} + \beta^2\tilde{C}), $$
so that $\mathrm{SNR}(\beta) = N(\beta)/D(\beta)$ and $\mathrm{SNR}(0) = A/(\sigma_N^2\tilde{A})$.
Consider the difference numerator:
$$ \Delta(\beta) := N(\beta)\tilde{A} - A \frac{D(\beta)}{\sigma_N^2} = -2\beta(B\tilde{A}-A\tilde{B}) + \beta^2(C\tilde{A}-A\tilde{C}). $$
Clearly $\Delta(0)=0$. 
By our assumption, the coefficient of the linear term $K = -2(B\tilde{A}-A\tilde{B})$ is strictly non-zero. 
For values of $\beta$ sufficiently close to $0$, the linear term $K\beta$ strictly dominates the quadratic term $\beta^2(C\tilde{A}-A\tilde{C})$. 
Since $K \neq 0$, we can choose a $\hat{\beta}$ with a sufficiently small absolute value and the same sign as $K$. 
This guarantees that $\Delta(\hat{\beta}) \approx K\hat{\beta} > 0$.

Since $D(\hat{\beta}) > 0$ and $\tilde{A} > 0$, the inequality $\Delta(\hat{\beta}) > 0$ directly implies
$$ N(\hat{\beta})\tilde{A} > A \frac{D(\hat{\beta})}{\sigma_N^2} \implies \frac{N(\hat{\beta})}{D(\hat{\beta})} > \frac{A}{\sigma_N^2\tilde{A}}. $$
Thus, $\mathrm{SNR}(\hat{\beta}) > \mathrm{SNR}(0)$.
\end{proof}

In PFGNet, we learn a per-channel $\beta$ in each branch and share it across spatial locations via backpropagation.
Since local spectral content varies across pixels, such as motion versus static regions, pixel-wise softmax fusion enables spatially varying optimal denoising, while the channel-wise $\beta$ provides learnable center suppression, outperforming any fixed global suppression.
Moreover, subtracting $\beta H_S$ can be beneficial: by Lemma~1, there exists a coefficient $\hat{\beta} \neq 0$ such that the filter $H_L - \hat{\beta} H_S$ achieves a strictly higher SNR than the plain large-kernel filter $H_L$.

This is empirically corroborated in Figure~\ref{fig:beta}.
On Moving MNIST, learned channel-wise $\tanh(\beta)$ values are symmetrically distributed around zero across all branches ($K{=}9,15,31$), with significant mass at $\pm 1$, indicating that both enhancement and suppression of peripheral responses are utilized to model simple digit motion.
On TaxiBJ, a complex traffic dataset, $\tanh(\beta)$ is centered at zero with reduced variance for larger kernels, reflecting balanced and mild center modulation to preserve rich mid-frequency flow patterns.
Notably, the symmetric usage of positive and negative $\beta$ on Moving MNIST and the near-zero mean on TaxiBJ are consistent with our SNR analysis: the learned channel-wise $\beta$ adapts in sign and magnitude to local spectral statistics, supporting the view that pixel-wise softmax gating together with channel-wise suppression improves denoising flexibility in practice.

\begin{figure}[t]
  \centering
  \begin{minipage}{0.485\columnwidth}
    \centering
    \includegraphics[width=\linewidth,height=0.56\linewidth]{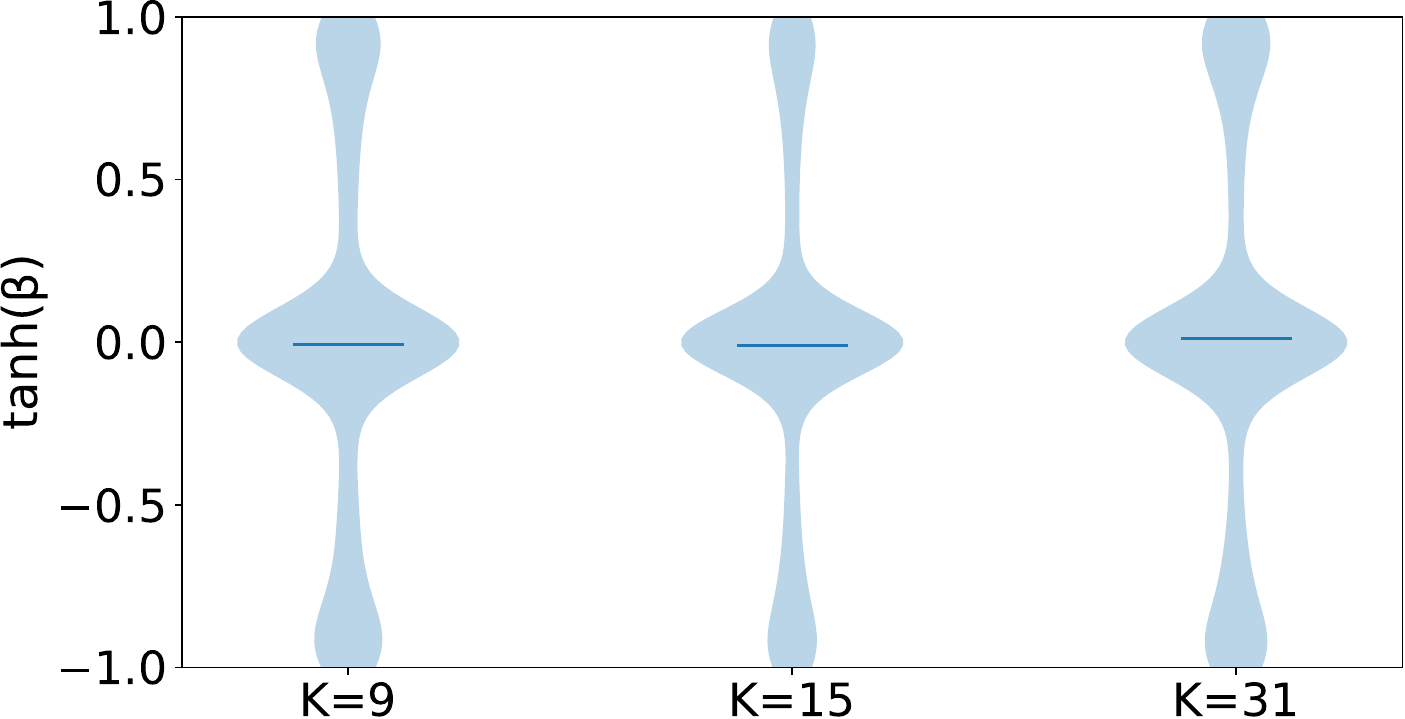}
    \caption*{(a) tanh($\beta$) by branch distribution on Moving MNIST}
  \end{minipage}\hfill
  \begin{minipage}{0.48\columnwidth}
    \centering
    \includegraphics[width=\linewidth,height=0.56\linewidth]{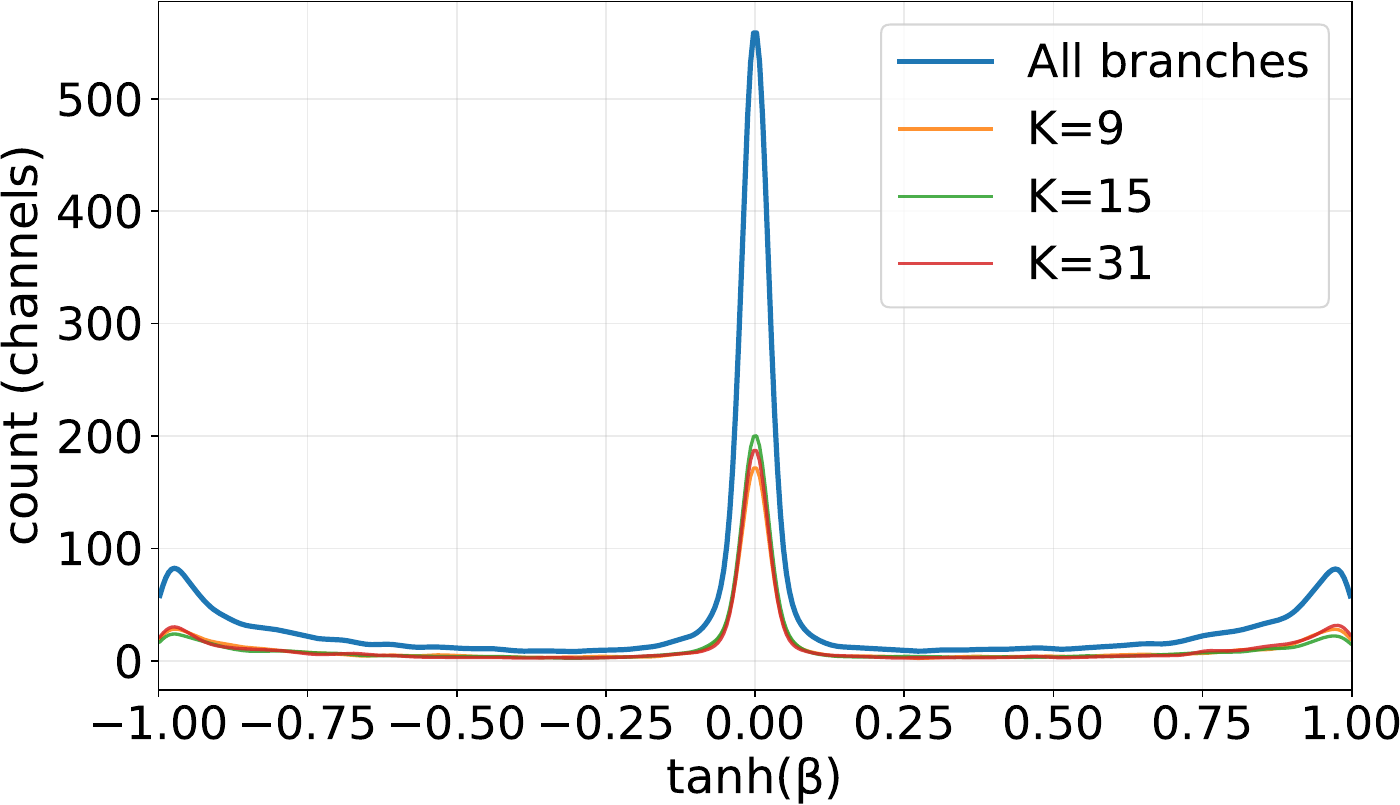}
    \caption*{(b) channel-wise tanh($\beta$) value distribution on Moving MNIST}
  \end{minipage}

  \vspace{2mm}

  \begin{minipage}{0.485\columnwidth}
    \centering
    \includegraphics[width=\linewidth,height=0.56\linewidth]{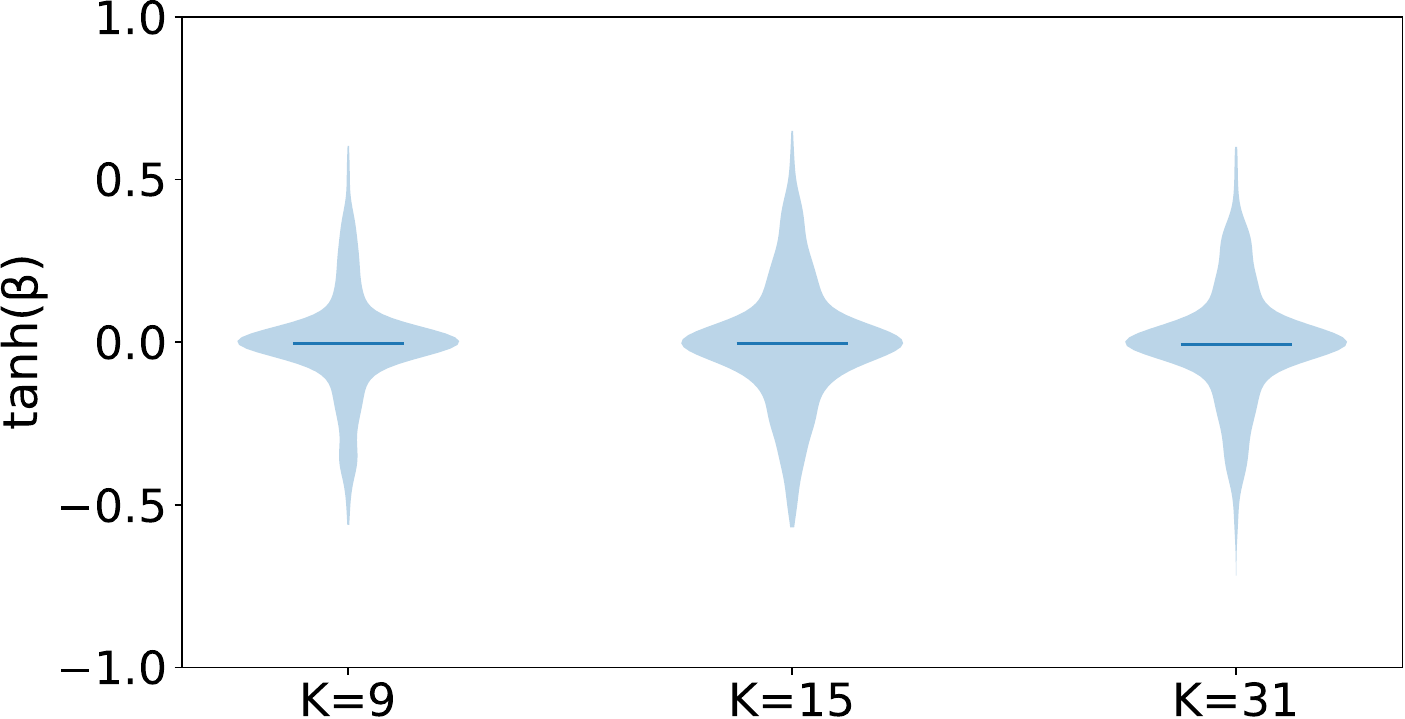}
    \caption*{(c) tanh($\beta$) by branch distribution on TaxiBJ}
  \end{minipage}\hfill
  \begin{minipage}{0.48\columnwidth}
    \centering
    \includegraphics[width=\linewidth,height=0.56\linewidth]{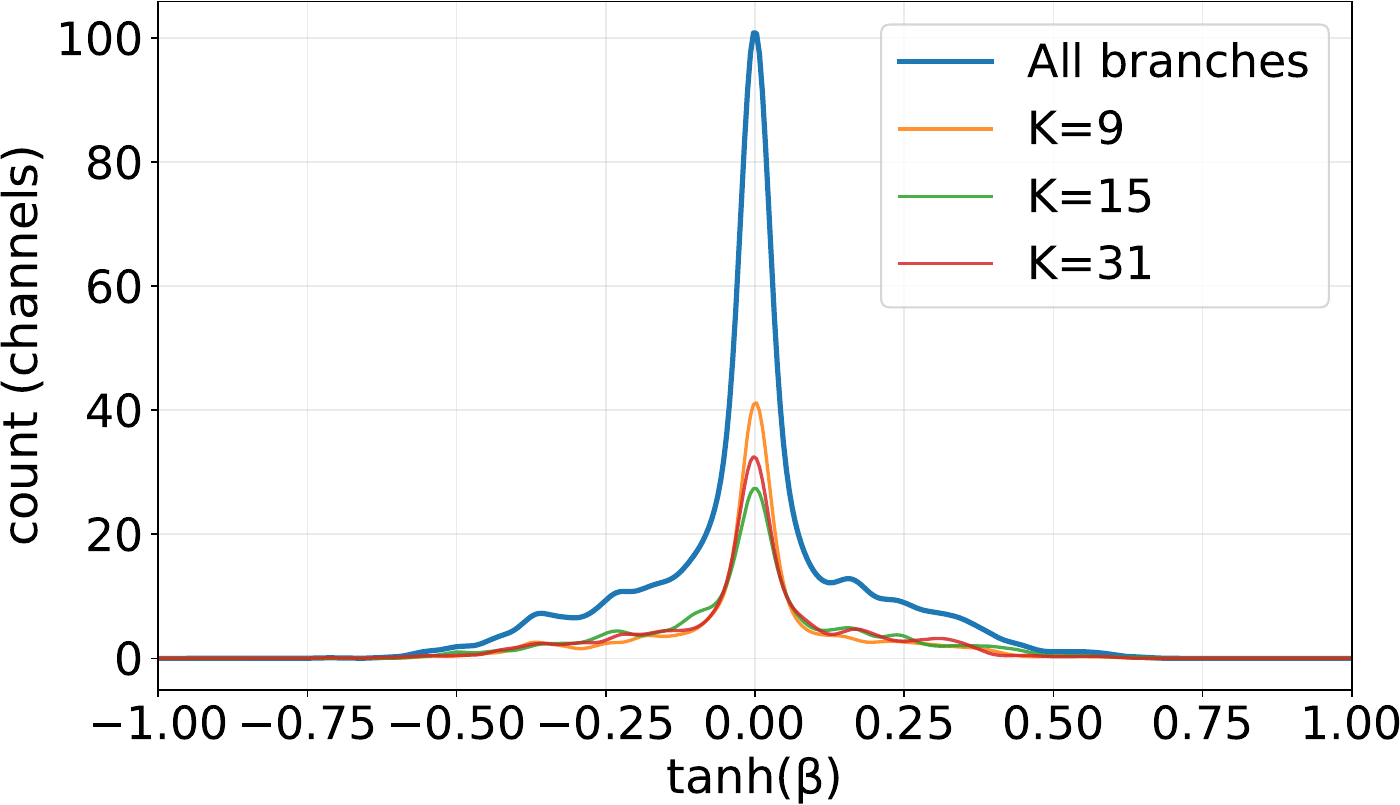}
    \caption*{(d) channel-wise tanh($\beta$) value distribution on TaxiBJ}
  \end{minipage}

  \vspace{1mm}
  \caption{Visualization of $\tanh(\beta)$ on Moving MNIST and TaxiBJ datasets.
  Each pair shows (a,c) the branch-wise $\tanh(\beta)$ distributions and (b,d) the smoothed channel-wise count curves,
where the horizontal axis denotes $\tanh(\beta)$ values and the vertical axis indicates the number of channels.}
  \label{fig:beta}
  \vspace{-3mm} 
\end{figure}

\section{Metric Definitions}

\begingroup
\allowdisplaybreaks

\paragraph{Notation}
\noindent
Let the prediction and ground truth be \(\hat{\mathbf{Y}}, \mathbf{Y} \in \mathbb{R}^{N \times T \times C \times H \times W}\). Dimensions: \(N\) is the batch size, \(T\) is the temporal length, \(C\) is the number of channels, and \(H, W\) are the spatial sizes. An element \(\hat{\mathbf{Y}}_{n,t,c,h,w}\) (or \(\mathbf{Y}_{n,t,c,h,w}\)) is the value at batch index \(n\), time step \(t\), channel \(c\), and spatial location \((h,w)\). Define the per-frame spatial size (including channels) as \(S = C H W\). All scalar metrics are averaged over the batch and time dimensions \((N, T)\).

\subsection{Mean Squared Error (MSE)}

\noindent\textbf{Non-spatially normalized}
{\footnotesize\[
\begin{aligned}
\mathrm{MSE}
&=\sum_{c=1}^{C}\sum_{h=1}^{H}\sum_{w=1}^{W}
\Biggl(\frac{1}{NT}\sum_{n=1}^{N}\sum_{t=1}^{T}
\bigl(\hat{\mathbf{Y}}_{n,t,c,h,w}-\mathbf{Y}_{n,t,c,h,w}\bigr)^{2}\Biggr).
\end{aligned}
\]}

\noindent\textbf{Spatially normalized}
\[
\begin{gathered}
\mathrm{MSE}_{\text{norm}} = \\
\frac{1}{NT}\sum_{n=1}^{N}\sum_{t=1}^{T}
\Biggl(\frac{1}{S}\sum_{c=1}^{C}\sum_{h=1}^{H}\sum_{w=1}^{W}
\bigl(\hat{\mathbf{Y}}_{n,t,c,h,w}-\mathbf{Y}_{n,t,c,h,w}\bigr)^{2}\Biggr).
\end{gathered}
\]

In practice, the spatially normalized MSE is used during training and validation,
while the non-spatially normalized version is reported during testing.

\subsection{Mean Absolute Error (MAE)}

{\footnotesize\[
\begin{aligned}
\mathrm{MAE}
&=\sum_{c=1}^{C}\sum_{h=1}^{H}\sum_{w=1}^{W}
\Biggl(\frac{1}{NT}\sum_{n=1}^{N}\sum_{t=1}^{T}
\bigl|\hat{\mathbf{Y}}_{n,t,c,h,w}-\mathbf{Y}_{n,t,c,h,w}\bigr|\Biggr).
\end{aligned}
\]}

\subsection{Peak Signal-to-Noise Ratio (PSNR)}

\noindent
Each pixel value \(a\) is converted to 8-bit as \(\tilde{a} = \mathrm{uint8}(255\,a)\).
For each sample frame \((n,t)\), the per-pixel mean squared error is
\[
\begin{aligned}
\mathrm{MSE}_{n,t}
&=\frac{1}{S}\sum_{c=1}^{C}\sum_{h=1}^{H}\sum_{w=1}^{W}
\bigl(\tilde{\hat{\mathbf{Y}}}_{n,t,c,h,w}-\tilde{\mathbf{Y}}_{n,t,c,h,w}\bigr)^{2}.
\end{aligned}
\]
With the dynamic range upper bound \(I_{\max}=255\), the frame-level and final PSNR are
\[
\begin{aligned}
\mathrm{PSNR}_{n,t}
&=20\log_{10}(I_{\max})-10\log_{10}\!\left(\mathrm{MSE}_{n,t}\right), \\
\mathrm{PSNR}
&=\frac{1}{NT}\sum_{n=1}^{N}\sum_{t=1}^{T}\mathrm{PSNR}_{n,t}.
\end{aligned}
\]

\subsection{Structural Similarity (SSIM)}

\noindent
For each frame, skimage.metrics.structural\_similarity is applied in the \([0,1]\) floating domain and averaged over \((N,T)\).
The frame-level SSIM uses the standard form.
With a Gaussian window \(G\),
\[
\begin{gathered}
\mu_x=G*x,\quad \mu_y=G*y,\quad
\sigma_x^{2}=G*(x^{2})-\mu_x^{2},\\
\sigma_y^{2}=G*(y^{2})-\mu_y^{2},\quad
\sigma_{xy}=G*(xy)-\mu_x\mu_y,
\end{gathered}
\]
where \(x\) and \(y\) denote two corresponding image patches, and
\(\mu_x, \mu_y, \sigma_x^{2}, \sigma_y^{2}, \sigma_{xy}\) are their local means,
variances, and covariance computed using the Gaussian window.
The SSIM for a pair of patches is defined as
\[
\mathrm{SSIM}_{\text{patch}}(x,y)
=
\frac{(2\mu_x\mu_y+C_{1})(2\sigma_{xy}+C_{2})}
{(\mu_x^{2}+\mu_y^{2}+C_{1})(\sigma_x^{2}+\sigma_y^{2}+C_{2})},
\]
where the dynamic range upper bound is \(I_{\max}=1\) and
\(C_{1}=(0.01\,I_{\max})^{2}\),
\(C_{2}=(0.03\,I_{\max})^{2}\).
The SSIM of an entire frame \(\hat{\mathbf{Y}}_{n,t},\mathbf{Y}_{n,t}\) is obtained by averaging over all overlapping patches:
\[
\mathrm{SSIM}_{\text{frame}}\!\left(\hat{\mathbf{Y}}_{n,t},\mathbf{Y}_{n,t}\right)
= \frac{1}{|\mathcal{P}|}
  \sum_{(x,y)\in\mathcal{P}} 
  \mathrm{SSIM}_{\text{patch}}(x,y),
\]
where \(\mathcal{P}\) denotes the set of all patches.
The final SSIM is
\[
\mathrm{SSIM}
=\frac{1}{NT}\sum_{n=1}^{N}\sum_{t=1}^{T}
\mathrm{SSIM}_{\text{frame}}\!\left(\hat{\mathbf{Y}}_{n,t},\mathbf{Y}_{n,t}\right).
\]

\section{Additional experimental results}
\label{sec:rationale}

Furthermore, we extend our evaluation to the Moving Fashion-MNIST (MFMNIST) dataset, with baseline results obtained from OpenSTL~\cite{tan2023openstl} under identical experimental settings. Figure~\ref{fig:mfmnist_qual} presents qualitative results of PFGNet. As shown in Table~\ref{tab:performance_comparison_more}, PFGNet achieves the \emph{highest accuracy among all recurrent-free models} 
(MSE: 23.55, SSIM: 0.9024) while maintaining moderate complexity (41.3M parameters, 15.2G FLOPs). 
Notably, compared to recurrent-based PredRNN (116.0G FLOPs) and PredRNN++ (171.7G FLOPs), 
PFGNet delivers \emph{significantly lower computational cost} 
with comparable predictive performance, 
highlighting its efficiency in spatiotemporal forecasting.

\begin{table}[htbp]
\scriptsize
\centering
\caption{Quantitative comparison on the Moving FMNIST dataset.}
\begin{tabular}{c c c c c c}
\toprule
Method & Params & FLOPs & MSE $\downarrow$ & MAE $\downarrow$ & SSIM $\uparrow$ \\
\midrule
\multicolumn{6}{c}{Recurrent-based Methods} \\
\midrule
ConvLSTM-S~\cite{shi2015convolutional} & 15.0M & 56.8G  & 28.87 & 113.20 & 0.8793 \\
ConvLSTM-L~\cite{shi2015convolutional} & 33.8M & 127.0G & 25.51 & 104.85 & 0.8928 \\
PredNet~\cite{lotter2017deep}          & 12.5M & 8.6G   & 185.94 & 318.30 & 0.6713 \\
PhyDNet~\cite{guen2020disentangling}   & \textbf{3.1M}  & 15.3G  & 34.75  & 125.66 & 0.8567 \\
PredRNN~\cite{wang2017predrnn}         & 23.8M & 116.0G & \underline{22.01}  & \textbf{91.74}  & \underline{0.9091} \\
PredRNN++~\cite{wang2018predrnn++}     & 38.6M & 171.7G & \textbf{21.71}  & \underline{91.97}  & \textbf{0.9097} \\
MIM~\cite{wang2019memory}              & 38.0M & 179.2G & 23.09  & 96.37  & 0.9043 \\
MAU~\cite{chang2021mau}                & 4.5M  & 17.8G  & 26.56  & 104.39 & 0.8722 \\
E3D-LSTM~\cite{wang2018eidetic}        & 51.0M & 298.9G & 35.35  & 110.09 & 0.8722 \\
PredRNNv2~\cite{wang2022predrnn}      & 23.9M & 116.6G & 24.13  & 97.46  & 0.9004 \\
DMVFN~\cite{hu2023dynamic}            & \underline{3.5M}  & \textbf{0.2G}   & 118.32 & 220.02 & 0.7572 \\
\midrule
\multicolumn{6}{c}{Recurrent-free Methods} \\
\midrule
SimVP~\cite{gao2022simvp}              & 58.0M & 19.4G  & 30.77  & 113.94 & 0.8740 \\
SimVPv2~\cite{tan2211simvp}            & 46.8M & 16.5G  & 25.86  & 101.22 & 0.8933 \\
TAU~\cite{tan2023temporal}             & 44.7M & 16.0G  & 24.24  & 96.72  & 0.8995 \\
ViT~\cite{dosovitskiy2021an}           & 46.1M & 16.9G  & 31.05  & 115.59 & 0.8712 \\
Swin Transformer~\cite{liu2021swin}    & 46.1M & 16.4G  & 28.66  & 108.93 & 0.8815 \\
Uniformer~\cite{li2022uniformer}       & 44.8M & 16.5G  & 29.56  & 111.72 & 0.8779 \\
MLP-Mixer~\cite{tolstikhin2021mlp}     & 38.2M & 14.7G  & 28.83  & 109.51 & 0.8803 \\
ConvMixer~\cite{trockman2023patches}   & 3.9M  & \underline{5.5G}   & 31.21  & 115.74 & 0.8709 \\
Poolformer~\cite{yu2022metaformer}     & 37.1M & 14.1G  & 30.02  & 103.39 & 0.8750 \\
ConvNeXt~\cite{liu2022convnet}         & 37.3M & 14.1G  & 26.41  & 102.56 & 0.8908 \\
VAN~\cite{guo2023visual}               & 44.5M & 16.0G  & 31.39  & 116.28 & 0.8823 \\
HorNet~\cite{rao2022hornet}            & 45.7M & 16.3G  & 29.19  & 110.17 & 0.8798 \\
MogaNet~\cite{moga}                    & 46.8M & 16.5G  & 25.14  & 99.69  & 0.8963 \\
PFGNet(Ours)                    & 41.3M & 15.2G & 23.55 & 94.78 & 0.9024 \\
\bottomrule
\end{tabular}
\label{tab:performance_comparison_more}
\vspace{-2mm} 
\end{table}

\begin{figure}[htbp]
\centering
\includegraphics[width=\linewidth]{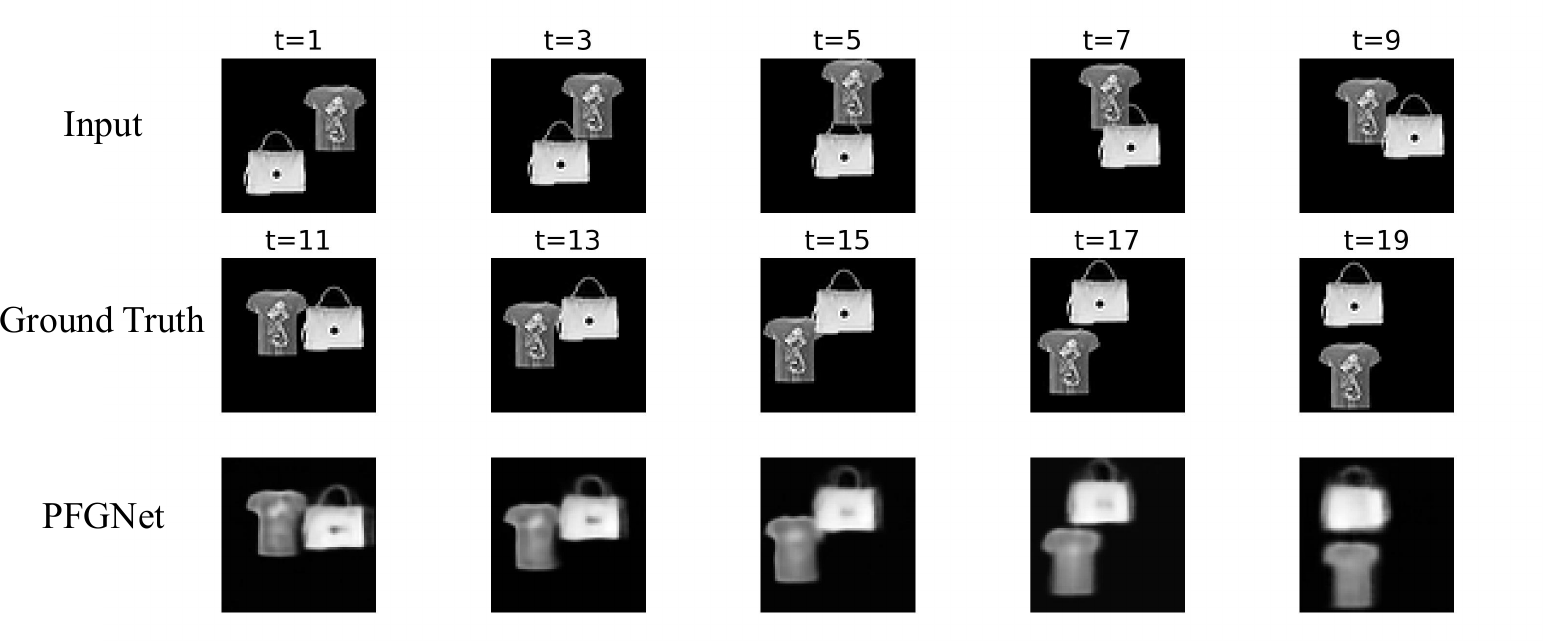}
\caption{
Qualitative results of PFGNet on Moving FMNIST. 
}
\label{fig:mfmnist_qual}
\vspace{-2mm} 
\end{figure}

\section{Additional Ablation Studies}

\begin{figure}[htbp]
  \centering
  \includegraphics[width=0.65\columnwidth]{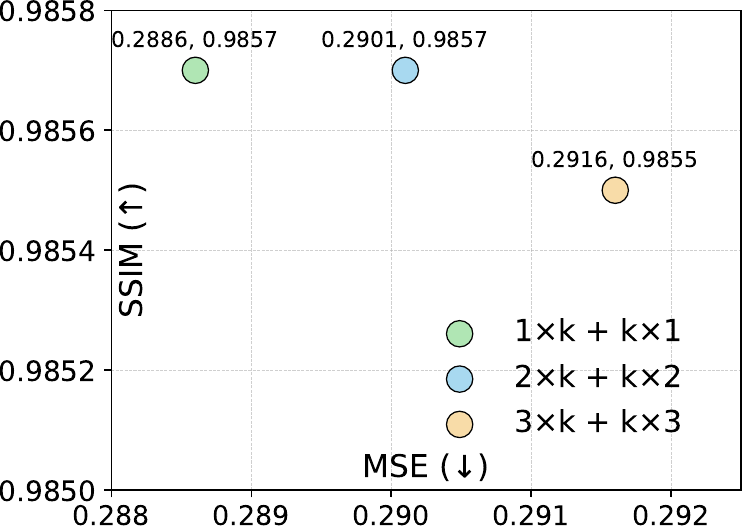}
  \caption{Ablation on asymmetric convolution ($n\times k + k\times n$).}
  \label{fig:ab3}
\end{figure}

As shown in Figure~\ref{fig:ab3}, we further analyze the effect of different $n$ settings in the asymmetric convolution ($n\times k + k\times n$) on the TaxiBJ dataset. The results show that varying $n$ brings only minor changes in performance; setting $n{=}1$ achieves nearly the same accuracy while reducing computational complexity. Therefore, we adopt $n{=}1$ as a compact and efficient default configuration in our experiments.

\begin{table}[tbp]
    \centering
    \small
    \caption{Ablation on the number of PFG blocks $N_t$ on Moving MNIST and TaxiBJ.}
    \label{tab:ablation_nt}
    {%
    \setlength{\tabcolsep}{3.15pt}
    \begin{tabular}{cc|ccccc}
        \hline
        Dataset & $N_t$ & Params & FLOPs & MSE $\downarrow$ & MAE $\downarrow$ & SSIM $\uparrow$ \\
        \hline
        \multirow{5}{*}{\shortstack{Moving MNIST\\\small(100 epochs)}}
            & 2  & 11.0M & 7.4G & 34.70  & 93.93 & 0.9201  \\
            & 4  & 21.1M & 10.0G & 29.74 & 83.38 & 0.9326  \\
            & 6  & 31.2M & 12.6G & 28.33  & 80.14  & 0.9363  \\
            & 8  & 41.3M & 15.2G & 27.61 & 78.54  & 0.9381  \\
            & 10 & 51.3M & 17.7G & \textbf{27.21}  & \textbf{77.57} & \textbf{0.9391}  \\
        \hline
        \multirow{5}{*}{TaxiBJ}
            & 2  & 0.53M  & 0.21G  & 0.3669  & 15.88  & 0.9826  \\
            & 4  & 0.97M  & 0.33G  & 0.3216  & 15.21  & 0.9842  \\
            & 6  & 1.41M  & 0.44G  & 0.3130  & 14.99  & 0.9851  \\
            & 8  & 1.86M  & 0.55G  & \textbf{0.2881}  & \textbf{14.75}  & \textbf{0.9857}  \\
            & 10 & 2.30M  & 0.67G  & 0.2919  & 14.78  & 0.9853  \\
        \hline
    \end{tabular}%
    }
\end{table}

\begin{table}[tbp]
    \centering
    \small
    \caption{Ablation on the number of PFG blocks $N_t$ on KTH (10$\rightarrow$20).}
    \label{tab:ablation_nt_kth}
    {%
    \setlength{\tabcolsep}{3.3pt}
    \begin{tabular}{cc|cccc}
        \hline
        Dataset & $N_t$ & Params & FLOPs & SSIM $\uparrow$ & PSNR $\uparrow$ \\
        \hline
        \multirow{5}{*}{KTH 10 $\rightarrow$ 20}
            & 2  & 10.8M & 57.1G & 0.908 & 33.88  \\
            & 4  & 20.8M & 98.5G & 0.910 & 34.01  \\
            & 6  & 30.9M & 0.14T & 0.911 & 34.10  \\
            & 8  & 41.0M & 0.18T & 0.909 & 34.10  \\
            & 10 & 51.1M & 0.22T & \textbf{0.912} & \textbf{34.24}  \\
        \hline
    \end{tabular}%
    }
\end{table}

In all main experiments, we set the number of PFG blocks $N_t$ to follow the same configuration as SimVP and its variants in OpenSTL to ensure a fair comparison; Tables~\ref{tab:ablation_nt} and \ref{tab:ablation_nt_kth} further investigate how varying $N_t$ from 2 to 10 affects model complexity and accuracy. As expected, Params and FLOPs grow linearly with $N_t$, since stacking more PFG blocks only introduces additional, structurally identical modules. On Moving MNIST, a larger $N_t$ consistently leads to lower MSE/MAE and higher SSIM, indicating that this dataset benefits from deeper temporal modeling and does not show overfitting within the tested range. On TaxiBJ, performance also improves as $N_t$ increases, but the best results are obtained at $N_t{=}8$, while $N_t{=}10$ brings slightly worse scores, suggesting a mild overfitting tendency and diminishing returns when the model becomes too deep. On KTH (10$\rightarrow$20), increasing $N_t$ from 2 to 10 generally leads to higher SSIM and PSNR, though the improvements become marginal once $N_t \ge 6$ and even show slight fluctuations, and the final training loss remains clearly below the validation loss, indicating that the motion structure is already well captured and extra depth mainly refines local details rather than improving generalization. In practice, choosing $N_t{=}4$ or $N_t{=}6$ recovers most of the gains while keeping the model compact, which aligns with the settings commonly used in OpenSTL. Overall, these results confirm that the OpenSTL default choice is well justified and provides a well-balanced trade-off between computational efficiency and predictive performance, and we therefore adopt this setting in our main experiments.

\begin{figure}[htbp]
  \centering
  \includegraphics[width=0.95\linewidth]{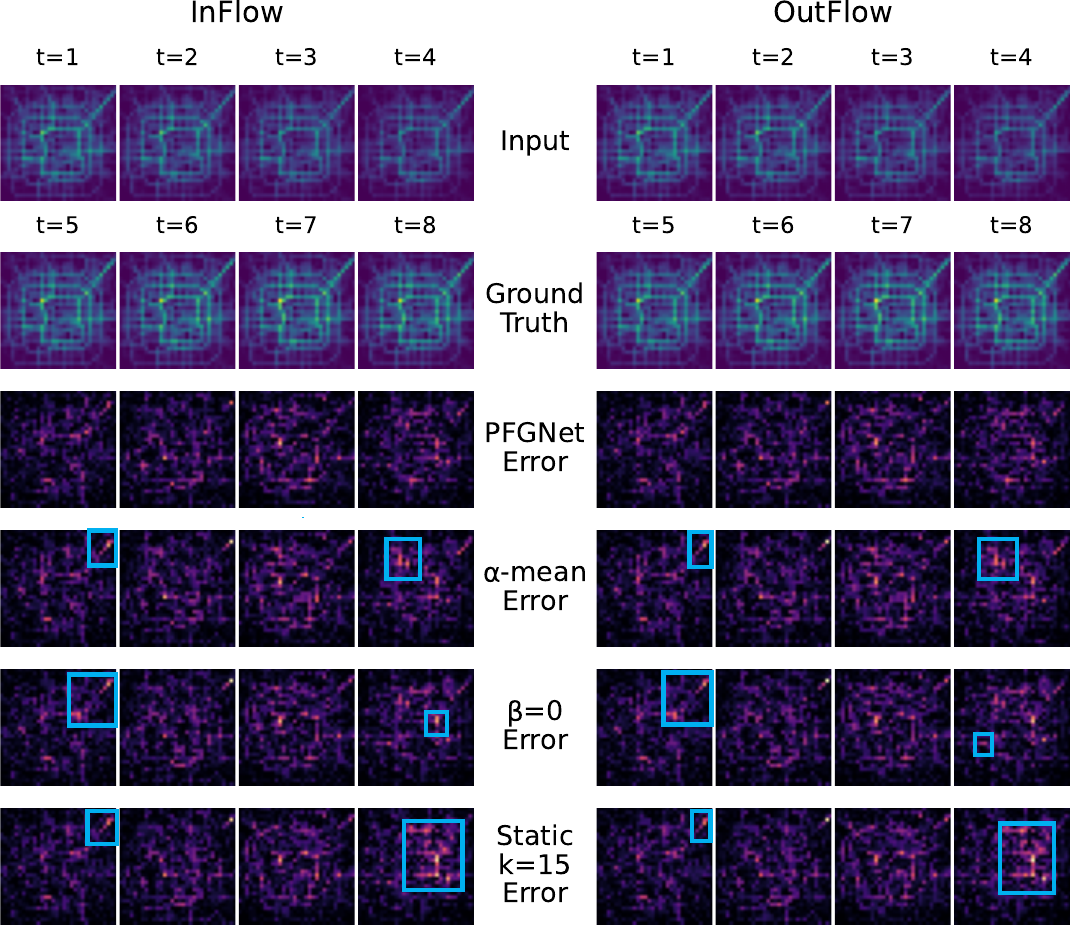}
  \caption{Qualitative visualizations on TaxiBJ.}
  \label{fig:taxibj_vis}
\end{figure}

Figure~\ref{fig:taxibj_vis} presents qualitative results on TaxiBJ. We compare PFGNet error maps ($|\hat{Y}-Y|$) against three ablations: uniform fusion ($\alpha$-mean), disabled center suppression ($\beta=0$), and static kernel ($k=15$). All maps share a unified scale. The ablations exhibit visibly larger and more diffuse errors (see highlights) compared to the sparse residuals of PFGNet. \textbf{This qualitative evidence corroborates the quantitative results in Sec.~4.3}, confirming that both adaptive gating and learnable center modulation are critical for refining structural details.

\begin{figure}[htbp]
  \centering
  \includegraphics[width=1.00\linewidth]{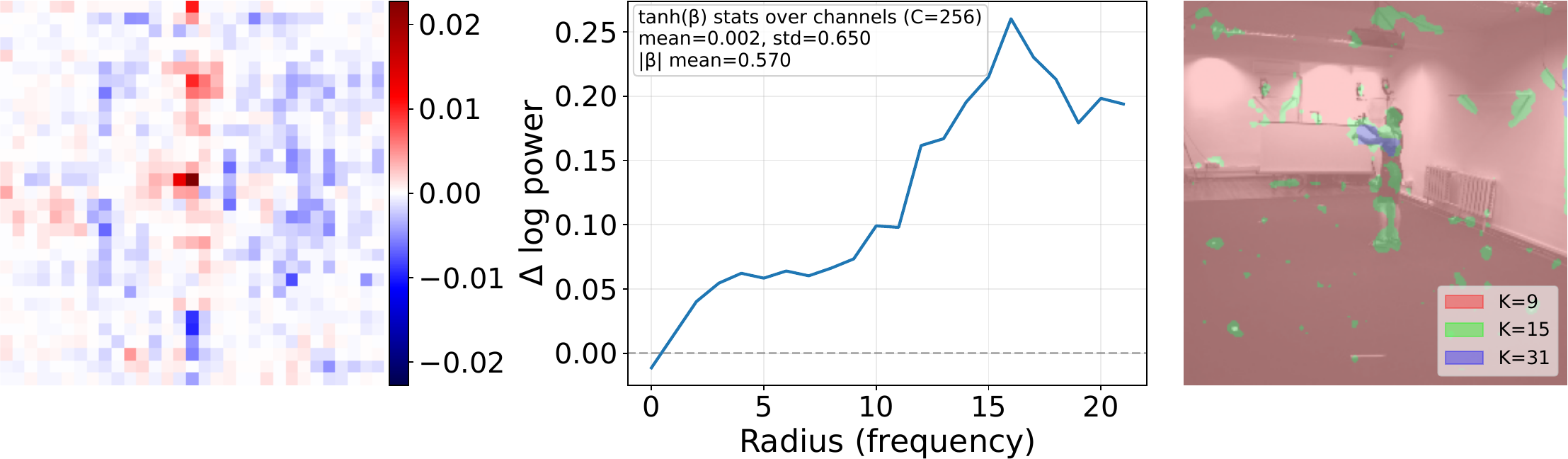}
  \caption{Mechanistic visualization of the learned PFGNet.}
  \label{fig:mechanistic_spectrum}
\end{figure}

\textbf{Layer-wise dissection on Human3.6M.}
As shown in Figure~\ref{fig:mechanistic_spectrum}, we select the information-dense Human3.6M benchmark and dissect one learned PFG module, reproducing the operator behavior from the spatial to the spectral domain.

\textbf{Spatial antagonistic structure (Left).}
The median effective kernel (\(K=31\) branch) exhibits a sign contrast between center and surround with an approximately ring-shaped tendency. This pattern is not manually imposed; it emerges from separable large-k peripheral aggregation and the learnable center--surround suppression term. The resulting shape is analogous to the classic Difference-of-Gaussians receptive-field model of retinal ganglion cells.

\textbf{Adaptive spectral re-weighting (Middle).}
We plot the radial log-power ratio for the \(K=31\) branch against \(\beta=0\) baseline ($\Delta\log P(r)=\log P_{\text{full}}-\log P_{\beta=0}$). Energy is weakened in the low-frequency region, while strengthened in the mid-to-high frequency region to enhance motion boundaries and structural changes. The curve flattens at the highest-frequency end, indicating that high-frequency noise is not over-amplified. With the statistics in the figure, \(\tanh(\beta)\) has near-zero mean and large standard deviation across channels, suggesting that the so-called suppression acts as channel-dependent bidirectional regulation: some channels suppress the center response, while others effectively compensate it. This flexibility aligns with antagonistic processing in biological vision.

\textbf{Spatially adaptive kernel selection (Right).}
We overlay the \(\arg\max\) map of the learned gating weights \(\alpha\), upsampled to the input resolution, on an input frame, visualizing the dominant kernel scale at each spatial location. The model prefers larger kernels around dynamic regions and motion boundaries, while assigning smaller kernels to smooth and static background, supporting input-adaptive receptive-field selection.


\end{document}